\documentclass{article}

\usepackage[preprint]{neurips_2026}
\usepackage[utf8]{inputenc} % allow utf-8 input
\usepackage[T1]{fontenc}    % use 8-bit T1 fonts
\usepackage{hyperref}       % hyperlinks
\usepackage{url}            % simple URL typesetting
\usepackage{booktabs}       % professional-quality tables
\usepackage{amsfonts}       % blackboard math symbols
\usepackage{nicefrac}       % compact symbols for 1/2, etc.
\usepackage{microtype}      % microtypography
\usepackage{xcolor}         % colors

\usepackage{amsmath,amssymb,amsthm} % math essentials
\usepackage{mathtools}              % fixes & extras for amsmath
\usepackage{bm}   % bold math symbols (vectors, tensors)
\newtheorem{theorem}{Theorem}
\newtheorem{lemma}{Lemma}

\newtheorem{definition}{Definition}
\theoremstyle{remark}
\newtheorem*{remark}{Remark}

\theoremstyle{definition}

\newtheorem{Proposition}{Proposition}

\usepackage{hyperref}
\usepackage{url}
\usepackage{float}  % in preamble
\usepackage{subcaption}
\usepackage{enumitem}

\usepackage{tikz}
\usepackage{calc}
\usetikzlibrary{arrows.meta, positioning, calc, bending}

\usepackage{cleveref}
\crefname{appendix}{Appendix}{Appendices}
\Crefname{theorem}{Theorem}{Theorems}
\Crefname{lemma}{Lemma}{Lemmas}
\Crefname{corollary}{Corollary}{Corollaries}
\Crefname{proposition}{Proposition}{Propositions}
\Crefname{claim}{Claim}{Claims}

\DeclareMathOperator*{\argmax}{arg\,max}
\usepackage{xspace}

\newcommand{\infoflow}{\textbf{InfoFlow}\xspace}

\title{InfoFlow: A Framework for Multi-Layer \\ Transformer Analysis
}

\author{%
  Penghao Yu\\
Department of Mathematics\\
National University of Singapore\\
\texttt{penghaoyu@u.nus.edu} \\
\And
Haotian Jiang \\
  Institute for Functional Intelligent Materials \\
  National University of Singapore \\
  \texttt{haotian@nus.edu.sg} \\
\And
Zeyu Bao \\
Department of Mathematics\\
National University of Singapore\\
\texttt{zeyu@u.nus.edu} \\
\And
Qianxiao Li \\
  Department of Mathematics \\
  Institute for Functional Intelligent Materials \\
  National University of Singapore \\
  \texttt{qianxiao@nus.edu.sg} 
}

\begin{document}

\maketitle

\begin{abstract}
    While the approximation properties of single-layer Transformer architectures have been studied in recent works, 
a rigorous theoretical understanding of the multi-layer setting remains limited. 
In this work, 
we establish that multi-layer Transformers possess fundamentally different approximation capabilities from single-layer ones: 
for certain retrieval tasks, 
any single-layer Transformer requires least $\Omega (\varepsilon^{-k})$ parameters to achieve precision $\varepsilon$, where $k$ grows linearly with sequence length $T$, 
whereas a two-layer Transformer with a single head per layer achieves the same approximation precision with at most $O (\varepsilon^{-1})$ parameters.
To understand this separation, 
we identify two structural mechanisms underlying multi-layer approximation.
Specifically, 
softmax attention can only efficiently retrieve the token attaining the maximum attention score, 
incurring exponential-in-length parameter cost for $k$-th largest retrieval with $k \geq 2$.
Moreover, the parameter cost of decoding coupled information scales with the size of the retrieved token set.
Motivated by these findings, 
we propose \infoflow, 
a framework for multi-layer Transformers.
The framework tracks an information set of accessible input positions at each token and layer, assigning an explicit approximation rate to each mode of information propagation.
This abstraction recovers known approximation bounds, 
remains consistent with experimental observations on trained networks, 
and yields concrete predictions in settings where direct theoretical analysis is currently intractable.
Our results provide a principled framework for reasoning about the approximation efficiency of multi-layer Transformers.
\end{abstract}

% \section{Introduction}
% {\bf \color{red} We call the surrogate model "Trans-Flow" for now}

% {\bf Part 1}: Talk about the importance of transformer, then summarize the known results on single-layer transformer, (e.g. universality results; previous work; "quality over quantity..."), then say that although much is known for single-layer transformer, the theoretical results/ understanding are still largely missing for multi-layer transformer.

% {\bf Part 2}: Furthermore, the theoretical analysis (of approximation rate) on multi-layer structure is very complicated. this is motivating us to develop a surrogate model along with "cost law" that aims to simplify the model; comply with theory; comply with experiment; can give predictions. 
% (Are there similar surrogate models in science/ ML. For example, chemistry is a "simplification" of quantum physics, and specifically explain how it is a "surrogate model". (e,g. the example of entropy... Its evaluation is not as exact as quantum physics, but simple enough to be solvable, and reliable enough to give good estimate.) 

% {\bf Part 3}: Our Contributions.

\section{Introduction}

The Transformer architecture \citep{vaswani2017.AttentionAllYou} has become the dominant framework for sequence modeling,
achieving remarkable empirical success across natural language processing \citep{devlin2019.BERTPretrainingDeep, brown2020.LanguageModelsAre},
computer vision \citep{dosovitskiy2020.ImageWorth16x16}.
Despite this empirical success, 
the theoretical understanding of its working principles remains incomplete. 
Progress has so far focused primarily on the single-layer setting:
universality results have been established by \cite{yun2020.AreTransformersUniversal,yun2020.ConnectionsAreExpressive},
explicit approximation rate estimates for single-layer Transformers were derived in \cite{jiang2024.ApproximationRateTransformer},
and the role of attention head count was characterized in 
\cite{yu2026.EffectAttentionHead},
showing that insufficient head count incurs parameter cost scaling at least as $O(\varepsilon^{-cT})$.
However,
the approximation theory of multi-layer Transformers remains largely open.
The analytical techniques developed for the single-layer case do not extend directly to multi-layer settings,
where the interplay between attention layers and feed-forward networks across depth introduces substantial additional complexity.
% To make this gap precise,
% we first establish that multi-layer Transformers possess fundamentally different approximation capabilities from single-layer ones. 
% Specifically, 
% we identify a retrieval target for which any single-layer Transformer requires parameter cost growing at least as $\Omega(\varepsilon^{-k})$ with $k$ increasing linearly with the sequence length $T$, 
% while a two-layer Transformer with one head per layer suffices to achieve approximation with parameter cost at most $O(\varepsilon^{-1})$. 
% The distinction in the rate result demonstrates that the multi-layer setting is not a quantitative extension of the single-layer one, 
% but reflects a qualitatively different approximation regime that requires new theoretical tools.
To make this gap precise,
we establish a depth separation result:
there exists a target such that any single-layer Transformer requires parameter cost growing with sequence length $T$,
while a two-layer Transformer with one head per layer achieves the same approximation precision $\varepsilon$ with parameter count at most $O(\varepsilon^{-1})$.

Motivated by the difficulty of direct theoretical analysis in the multi-layer setting,
we develop a framework,
which we call \textbf{InfoFlow},
in the spirit of effective theories in the physical sciences.
Just as thermodynamics predicts the phase transition of water without quantum mechanical analysis of individual molecules,
\infoflow provides a tractable abstraction of multi-layer Transformer computation that retains sufficient structure to yield principled predictions.
Rather than tracking hidden representations,
\infoflow associates to each token position $t$ and layer $l$ an information set $I(t,l)$,
defined as the collection of input positions whose information remains accessible,
together with a parameter cost law governing each mode of information propagation.
The design of \infoflow is grounded in two structural results established in this work:
softmax attention can only efficiently retrieve the token attaining the maximum attention score,
with parameter cost growing exponentially with $T$ for any other token,
and the cost of decoding aggregated information from a token state scales exponentially with the number of input tokens contributing to it.

We validate \infoflow both theoretically and experimentally.
On the theoretical side,
it recovers known single-layer approximation bounds and is consistent with the depth separation result.
On the experimental side,
it correctly predicts the intrinsic dimension phenomenon in two-layer Transformers and explains the approximation difficulty of higher-order retrieval tasks such as the triangle-center problem.
Our main contributions are as follows.
\textbf{I}: We establish a depth separation result between single and two-layer Transformers (Theorem ~\ref{theorem1}),
    a precise characterization of the retrieval limitation of softmax 
    attention (Theorem ~\ref{theorem2}),
    and an estimate of the parameter cost of decoding aggregated 
    information (Theorem ~\ref{theorem3}).
    \textbf{II}: We propose \textbf{InfoFlow},
    a framework comprising an information set update rule 
    and an associated parameter cost law,
    grounded in the above theoretical results and consistent with 
    experimental observations.
    \textbf{III}: We validate \infoflow empirically,
    demonstrating that it correctly predicts the intrinsic dimension 
    phenomenon and explains the approximation hardness of higher-order 
    retrieval tasks such as the triangle-center problem.

\section{Related Work}
The theoretical study of Transformer approximation properties has developed along two main directions:
universality results and approximation rate results. 
\cite{yun2020.AreTransformersUniversal} 
first established that Transformers are universal approximators of continuous sequence-to-sequence functions on compact domains,
subsequently extended to sparse attention matrices \citep{yun2020.ConnectionsAreExpressive} and to single self-attention layers with low-rank weight matrices \citep{kajitsuka2023.AreTransformersOne}.
Beyond universality, explicit Jackson-type approximation rate estimates for 
single-layer Transformers were derived in \cite{jiang2024.ApproximationRateTransformer}, 
where a complexity measure governing the low-rank pairwise coupling structure of the target was identified as the key quantity.
This was extended in \cite{yu2026.EffectAttentionHead}, 
which established both upper and lower approximation bounds,
% on the parameter complexity required for $\varepsilon$-approximation with particular emphasis on the role of the attention head count, 
showing that insufficient head count incurs parameter cost scaling 
% at least as $O(\varepsilon^{-cT})$ for some constant $c$ and 
with sequence length $T$. 
% In a related direction, 
% \cite{amsel2024.QualityQuantityAttention} showed that attention rank rather than head count can be the key factor in approximation capacity. 
% demonstrating that a nearest-neighbor retrieval target requires exponentially many low-rank heads to approximate even when the total parameter count is large.
% Approximation rates from a different perspective were developed in \cite{takakura2023.ApproximationEstimationAbility}, 
% who derived approximation and estimation bounds for infinite-length sequence-to-sequence functions characterized by smoothness.
% These results collectively advance the understanding of single-layer 
% Transformer approximation,
% yet they leave open the fundamental question of whether depth provides a 
% qualitative advantage in approximation efficiency,
% a question we address directly in this work.
% \paragraph{Expressivity and depth separation for multi-layer Transformers.}
\cite{hahn2020.TheoreticalLimitationsSelfAttention} showed that the computational power of self-attention is fundamentally limited,
and that modeling certain sequential dependencies requires the number of layers or heads to grow with input length.
More recently, 
~\cite{sanford2024.OnelayerTransformersFail,sanford2024.TransformersParallelComputation} established that one-layer Transformers cannot solve the induction head task without exponentially increasing model size, 
while two layers suffice with only logarithmic growth. 
% providing a concrete depth separation result for a specific retrieval mechanism. 
% The present work establishes a complementary depth separation result in a nonlinear continuous approximation setting, 
% showing that a retrieval target which is exponentially hard for single-layer Transformers admits an $O(\varepsilon^{-1})$ approximation by a two-layer Transformer. 
% To the best of our knowledge, 
% this constitutes the first depth separation result of this type in a nonlinear approximation rate setting.
% \paragraph{Information flow and mechanistic interpretability.}
A parallel line of work focuses on understanding multi-layer Transformer by tracking how information propagates through the network, 
rather than through formal approximation theory. 
\cite{elhage2021.MathematicalFrameworkTransformer} introduced the Transformer Circuits framework, 
which decomposes the residual stream into communication channels and identifies specific circuits responsible for particular computational behaviors. 
% ~\cite{olsson2022.IncontextLearningInduction} extended this empirically, 
% identifying induction heads as a general mechanism for in-context learning in trained models.
% These works aim to reverse-engineer the algorithms learned by specific trained models.
\infoflow, 
by contrast,
is a theoretical approximation tool that operates prior to training and yields quantitative predictions about approximation rates and task feasibility.
% \paragraph{Coarse-grained and effective models of Transformer computation.}
Several works have proposed simplified mathematical descriptions of Transformers aimed at tractable analysis.
A line of work interprets multi-layer Transformers as discretizations of continuous differential equations or mean-field interacting particle systems \citep{geshkovski2023.MathematicalPerspectiveTransformers, lu2020.UnderstandingImprovingTransformer},
where the token sequence evolves according to a continuous-time flow governed by the attention mechanism.
These approaches abstract away the discrete layer structure by taking continuous limits of the hidden representations,
and focus on analyzing the dynamical properties of the resulting flow.
\infoflow is developed in the same spirit,
but abstracts multi-layer Transformer computation through discrete information sets rather than continuous limits,
and focuses on predicting approximation rates and task feasibility rather than analyzing dynamical properties.

\section{Preliminaries and Mathematical Formulations}

We introduce the multi-layer Transformer architecture and establish the 
notation used throughout the paper.
% The single-layer formulation introduced in \citep{yu2026.EffectAttentionHead} 
% admits a natural extension to the multi-layer setting.
We consider the input space $\mathcal{X}^T = \{x(s) \in [0,1]^d : s \in [T]\}$,
where $[T] = \{1,\dots,T\}$ and $T$ denotes the sequence length.
The output is a single vector $y \in \mathbb{R}^o$.
Each token is mapped into $\mathbb{R}^{E_1}$ by a trainable encoder.
When positional encoding is used,
the encoder takes the form $P_\phi : [0,1]^d \times [T] \to \mathbb{R}^{E_1}$,
$(x,s) \mapsto P_\phi(x,s)$,
depending on both token content and position,
with standard choices including additive encoding
$P_\phi(x(s),s) = \mathrm{Emb}(x(s)) + p(s)$
and augmented encoding
$P_\phi(x(s),s) = (\mathrm{Emb}(x(s)), p(s))$.
Without positional encoding,
we use $P_\phi : [0,1]^d \to \mathbb{R}^{E_1}$, $x \mapsto P_\phi(x)$.
In both cases, we append a trainable classification token 
$c_1 \in \mathbb{R}^{E_1}$ and set
%\begin{equation}
%X_1[T] = \{\hat{x}(1),\dots,\hat{x}(T),\hat{c}_0\} 
%\in \mathbb{R}^{E_1 \times (T+1)}.
%\end{equation}
$ X_1[T] = \{P_\phi(X_T)_{t=1}^{T},c_1\} 
\in \mathbb{R}^{E_1 \times (T+1)}$.
   
For an $L$-layer Transformer,
we denote by $h_l$ the number of heads in the $l$-th layer,
$n_l$ the per-head embedding dimension,
and $E_l = h_l n_l$ the total embedding dimension of the $l$-th layer.
For $l = 1,\dots,L-1$,
the $l$-th attention layer takes 
$X_l[T] = \{x_l(1),\dots,x_l(T),c_l\} \in \mathbb{R}^{E_l \times (T+1)}$
as input and produces
$X_l'[T] = \{x_l'(1),\dots,x_l'(T),c_l'\} \in \mathbb{R}^{E_l \times (T+1)}$,
where
\begin{equation}
\begin{aligned}
x_l'(t) &= x_l(t) + W_{O,l} \operatorname{Concat}_{i=1}^{h_l}
\Bigl(\sum_{s=1}^{T} \sigma\bigl[(W_{Q,i,l}x_l(t))^\top 
W_{K,i,l}x_l(s)\bigr] W_{V,i,l}x_l(s)\Bigr), \\
c_l' &= c_l + W_{O,l} \operatorname{Concat}_{i=1}^{h_l}
\Bigl(\sum_{s=1}^{T} \sigma\bigl[(W_{Q,i,l}c_l)^\top 
W_{K,i,l}x_l(s)\bigr] W_{V,i,l}x_l(s)\Bigr),
\end{aligned}
\label{eq:attention}
\end{equation}
with query/key projections $W_{Q,i,l}, W_{K,i,l} \in \mathbb{R}^{n_l \times E_l}$,
value projection $W_{V,i,l} \in \mathbb{R}^{n_l \times E_l}$,
output projection $W_{O,l} \in \mathbb{R}^{E_l \times E_l}$,
and residual connections.
The softmax with scaling factor $\beta > 0$ is defined by
$
\sigma[\rho](s) = \frac{\exp(\beta\,\rho(s))}
{\sum_{t'=1}^T \exp(\beta\,\rho(t'))}.
\label{eq:softmax}
$
Between layers,
a two-layer feed-forward network 
$F_l : \mathbb{R}^{E_l} \to \mathbb{R}^{E_{l+1}}$
with hidden-layer width $w_l$ maps the outputs of the $l$-th attention 
layer to the inputs of the $(l+1)$-th layer via
$x_{l+1}(t) = F_l(x_l'(t))$ and $c_{l+1} = F_l(c_l')$, with sigmoidal activation functions that are $1$-Lipschitz.
The final layer takes
$X_L[T] = \{x_L(1),\dots,x_L(T),c_L\} \in \mathbb{R}^{E_L \times (T+1)}$
as input and produces the output
\begin{equation}
y = \hat{F}\!\Bigl(c_L + W_{O,L} \operatorname{Concat}_{i=1}^{h_L}
\Bigl(\sum_{s=1}^{T} \sigma\bigl[(W_{Q,i,L}c_L)^\top 
W_{K,i,L}x_L(s)\bigr] W_{V,i,L}x_L(s)\Bigr)\Bigr),
\label{eq:final-layer}
\end{equation}
where $\hat{F}: \mathbb{R}^{E_L} \to \mathbb{R}^o$ is a feed-forward block.

\section{Structural Mechanisms of Multi-Layer Transformer Approximation}\label{sec:mechanisms}
We now turn to the approximation properties of the multi-layer Transformer defined above.
The central questions are how the parameter cost of $\varepsilon$-approximation scales with the architectural parameters $T$, $L$, $E_l$, $h_l$,
and which complexity measures of the target determine this scaling.
We show that these questions cannot be addressed by a direct extension of single-layer results,
and then identify the structural mechanisms underlying multi-layer approximation.

\subsection{Depth Separation Between Single-Layer and 
Multi-Layer Transformers}

% A natural starting point for the analysis of multi-layer Transformers is to 
% ask whether their approximation properties can be obtained by extending 
% single-layer results.
% Building on the framework of \cite{yu2026.EffectAttentionHead},
% we show that this is not the case: single-layer and two-layer Transformers 
% exhibit fundamentally different approximation behaviors,
% as formalized in the following theorem.

A natural starting point for the analysis of multi-layer Transformers is to ask whether their approximation properties can be obtained by extending single-layer results.
We show that this is not the case.
In fact,
single-layer and two-layer Transformers exhibit fundamentally different approximation behaviors,
as the following theorem demonstrates.

\begin{theorem}[Informal]\label{theorem1}
There exists a target $F : \mathcal{X}^T \to \mathbb{R}$ 
satisfying the following.
\begin{enumerate}
    \item A single-layer Transformer  
    with $H$ heads and per-head dimension $n$ must incur parameter count 
    at least $\Omega(\varepsilon^{-k(T)})$ with 
    $k(T) = \frac{T-2H}{H(n+1)+1} - 1$ to $\varepsilon$-approximate $F$.
    \item For every $\varepsilon > 0$,
    there exists a two-layer Transformer with one head per layer 
    and per-head dimension $6$ with parameter count at most 
    $O(\varepsilon^{-1})$ that $\varepsilon$-approximates $F$.
\end{enumerate}
\end{theorem}

% We defer the formal statement and proof to Appendix ~\ref{app:thm1}.
% \Cref{theorem1} shows that the gap between single-layer and 
% two-layer approximation is not merely quantitative but qualitative:
% the exponent $k$ grows with $T$ for any fixed $H$ and $n$,
% meaning that single-layer parameter cost is unavoidably super-polynomial 
% in $1/\varepsilon$ for long sequences,
% while a fixed two-layer architecture achieves $O(\varepsilon^{-1})$ 
% regardless of $T$.
% This separation raises the question of which structural mechanisms enable multi-layer Transformers to achieve this advantage,
% which we examine in the following subsection.

We defer the formal statement and proof to Appendix ~\ref{app:thm1}.
The depth separation in Theorem~\ref{theorem1} reflects a fundamental architectural difference between single-layer and multi-layer Transformers.
A single-layer Transformer performs only one round of attention,
and can therefore compare each token against the sequence at most once.
A two-layer Transformer,
by contrast,
can perform comparisons in two stages:
the first layer allows each token to find its most relevant neighbor,
and the second layer aggregates these local results into a global answer.
For the target $F(X_T) = \min_{1 \le s,t \le T} 2(1 + x(s)^\top x(t))$,
this two-stage structure allows the model to decompose the global 
minimization over $T^2$ pairs into two successive minimizations 
each of order $T$,
achieving $O(\varepsilon^{-1})$ approximation regardless of $T$.
A single-layer Transformer has no such decomposition available,
and must instead rely on very wide feed-forward layers to resolve the 
remaining comparisons,
requiring parameter count at least $\Omega(\varepsilon^{-k(T)})$ where 
$k(T)$ grows linearly with $T$.

\subsection{Information Propagation Mechanisms in Attention Layers}

To understand the mechanism behind the depth separation established above,
we examine how information propagates through the attention layers of a multi-layer Transformer.
We identify three characteristic modes of information propagation,
each supported by theoretical analysis or experimental evidence.
These modes are not intended as an exhaustive classification of all possible attention behaviors. 
Rather, 
they capture the dominant mechanisms that emerge in practice and are amenable to theoretical characterization, forming the basis of the \infoflow framework introduced in Section~\ref{sec:infoflow}.

\paragraph{Mechanism I: Max-position retrieval.}
A first observation from two-layer Transformers is that the token state at position $t$ after one attention layer retains information from two sources.
First, the residual connection ensures that the original token $x(t)$ itself remains accessible from the attention output $x^{(1)}(t)$ by construction.
Second, 
experiments on trained two-layer Transformers show that the information of the token at the argmax position of each attention head,
namely $x(s_i(t))$ where $s_i(t) = \arg\max_{s} A_i(x(t), x(s))$,
can be recovered from $x^{(1)}(t)$.
Specifically, we measure the normalized recovery error 
$\mathbb{E}_{X_T}\mathbb{E}_t |g_i(x_1'(t)) - x(s_i(t))|^2 / |x(s_i(t))|^2$
for the best map $g_i$,
and find that training reduces this error by an order of magnitude compared to initialization across all configurations tested,
with ratios of trained to initialized error as low as $0.08$
(see Appendix~\ref{app:exp_max_retrieval} for full experimental details).
Furthermore, the following theorem establishes that softmax attention is fundamentally limited to efficient retrieval at the argmax position, 
and cannot efficiently retrieve any other token. This is further confirmed in the experiment in Appendix~\ref{app:exp_second_cannot}.

\begin{theorem}[Informal]\label{theorem2}
For a given $k$ with $2 \le k \ll T$,
if a softmax attention head with embedding dimension $n$ retrieves the information of the token with $k$-th largest attention score with precision $\varepsilon > 0$,
then it requires at least $\Omega\!\left(\varepsilon^{1 - \frac{T-k-1}{n+1}}\right)$ parameters in the feed-forward blocks.
\end{theorem}

We defer the formal statement and proof to Appendix~\ref{app:thm2}.
The experimental observation and Theorem~\ref{theorem2} together establish what we term the \emph{max-position retrieval mechanism}.
The residual connection and argmax concentration of softmax attention ensure that each token state retains information from its own position and from the position attaining the maximum attention score,
while retrieving information from any other position incurs a parameter cost growing exponentially in $T$.
This mechanism is a structural property of softmax attention that persists regardless of the projection matrices chosen.

\paragraph{Mechanism II: Global information aggregation.}
A single attention head may also aggregate information from the entire input sequence at one position,
enabling the model to handle relationships that depend on all tokens simultaneously.
As established in \citep{yu2026.EffectAttentionHead},
this global aggregation mode incurs a parameter cost of $O(\varepsilon^{-T/E_l})$ in the feed-forward block following the attention layer when the embedding dimension satisfies $E_l \le T$.
This cost grows rapidly with sequence length $T$,
reflecting the fundamental difficulty of compressing global sequence information into a fixed-dimensional token state,
and suggesting that global aggregation is only effective when the sequence length is not large relative to the embedding dimension.

\paragraph{Mechanism III: Specific position aggregation.}
When positional encoding is available,
a Transformer can aggregate information from a fixed set of positions determined by positional information alone,
independent of the sequence content.
To illustrate this mechanism,
we train a one-layer one-head Transformer with positional encoding on the target $F(X_T) = x(1) + x(2) + x(3)$ for sequences of length $T = 10$.
As shown in Figure~\ref{fig:pos-agg},
the CLS token's attention concentrates with weight $1/3$ on each of positions $1$, $2$, $3$ and weight $0$ on all remaining positions,
a pattern that is stable across samples regardless of token values.
This confirms that the model selects positions by index alone rather than by content,
a mode of information propagation that is unavailable without positional encoding,
as attention weights in that setting depend only on token content and thus cannot implement position-specific selection.
Full experimental details are given in Appendix~\ref{app:exp2}.

We demonstrate the three mechanisms in Figure ~\ref{fig:mechanisms}.
They together characterize the dominant ways in which attention layers move information across token positions.
This evidence motivates a framework that tracks which input positions contribute to each token state across positions and layers.
We formalize this framework in the next section.

\section{\infoflow: A framework for analyzing Multi-Layer Transformers}
\label{sec:infoflow}

Motivated by the three information propagation mechanisms identified in Section~\ref{sec:mechanisms},
we propose \infoflow,
a framework that replaces Transformer's hidden vector representations with discrete information sets,
together with a parameter cost law governing each mode of propagation.

\subsection{The \infoflow Model}
\label{sec: info flow}

Instead of tracking the actual hidden vectors,
we associate to each position $t \in \{1,\dots,T+1\}$ at layer $l$ an information set
$I(t,l) \subset \{1,\dots,T+1\}$,
which represents the collection of input positions whose information remains accessible from the token state $x_l'(t)$ at position $t$ after $l$ layers.
Here the $(T+1)$-th position corresponds to the classification token $c_l'$.
At the input layer,
each token contains information from its own position only,
so we set
\begin{equation}
I(t,0) = \{t\}, \quad 1 \le t \le T, \qquad I(T+1,0) = \varnothing.
\end{equation}
We now define the update rule from layer $l$ to layer $l+1$.
The information flow from $(t,l)$ to $(t,l+1)$ takes one of the following three forms,
each corresponding to a mechanism identified in Section~\ref{sec:mechanisms}.

\paragraph{(1) Max-position retrieval.}
Corresponding to the max-position retrieval mechanism,
we define for each head $i \in \{1,\dots,h_{l+1}\}$ the position 
at layer $l$ whose information contributes most to token $t$ 
at layer $l+1$ under head $i$ as
$s_{i,l}(t) = \argmax_{s \in \{1,\dots,T\}} A_{l+1,i}\bigl(X[I(t,l)],\, X[I(s,l)]\bigr)$,
where $A_{l+1,i}$ denotes the attention score function of the $i$-th 
head in layer $l+1$,
and $X[I]$ denotes the subsequence of $X_T$ at indices in $I$.
The information set update rule is then
$
I(t,l+1) \subset \bigcup_{i=1}^{h_{l+1}} I(s_{i,l}(t),l) \cup I(t,l).
$
The new token state uses only information from the argmax positions 
selected by each head and from its own position.

\paragraph{(2) Global information aggregation.}
Corresponding to the global aggregation mechanism,
one attention head aggregates information from the entire input 
sequence at position $t$,
so the information set after the update contains all input positions,
$
I(t,l+1) = \{1,\dots,T\}.
$
The new token state thus depends on the entire input sequence.

\paragraph{(3) Specific position aggregation.}
Corresponding to the specific position aggregation mechanism,
which requires positional encoding,
the attention head selects a fixed set of positions $I'(t,l+1)$ 
determined by the position $t$ and layer $l+1$ alone,
independent of the sequence content.
The information set update rule is then
$I(t,l+1) \subset \bigcup_{j \in I'(t,l+1)} I(j,l)$,

where $I'(t,l+1)$ is determined only by $t$ and $l+1$,
not by the values of the input sequence.

Having defined the three modes of information propagation,
we now specify what it means for an \infoflow model to learn a target.
For a given $a.e.$ differentiable target $F : \mathcal{X}^T \to \mathbb{R}$,
define the active index set at differentiable point $X_T$ as
\begin{equation}
I_F(X_T) = \Bigl\{t : \frac{\partial F}{\partial x(t)}\Big|_{X_T} \ne 0\Bigr\}.
\end{equation}
We say \infoflow model learns the target $F$ if $I_F(X_T) \subseteq I(T+1,L)$ for $X_T \in \mathcal{X}^T\,a.e.$ in Lebesgue measure,
where $I(T+1,L)$ is the information set of the classification token in the final layer.

\begin{remark}
We particularly focus on retrieval targets,
where $|I_F(X_T)|$ does not increase with $T$.
An illustrative example is shown in Appendix ~\ref{app: example}.
\end{remark}

% \jht{we put it separate because they come from one reasoning}

% {\color{red}  Another way is to write everything simpler, by assigning parameter cost $O(\varepsilon^{1-\frac{|I(t,l+1)|d}{E_{l+1}}})$ regardless of the type of update rule chosen. }

\subsection{Parameter Cost}
\label{sec: parameter_cost}

% \jht{Add Theorem 5 here , and explain parameter cost}
In practical Transformers,
information accumulated in the information set $I(t,l)$ is not represented explicitly but is embedded in the token state $x_l(t) \in \mathbb{R}^{E_l}$,
and recovering it requires nontrivial decoding by the feed-forward layers. We use the following theorem to quantify the parameter cost of this encoding-decoding process. 
\begin{theorem}[Informal]\label{theorem3}
    For index set $I$ of size $T_0$, if we use two-layer FFN $f_1$ to encode $X[I]$ in $[0,1]^n$ and $f_2$ to decode $\hat X[I]$ from $[0,1]^n$. If $\hat X[I] = f_2 \circ f_1 (X[I])$ is $\varepsilon$-close to $X[I]$ for all $X_T$, then $f_1$ and $f_2$ have parameter count of $O(\varepsilon^{1-\frac{T_0d}{n}})$.
\end{theorem}
We defer the formal statement and proof to Appendix~\ref{app:thm3}. Following Theorem~\ref{theorem3}, it requires $O(\varepsilon^{1-\frac{|I(t,l+1)|d}{E_{l+1}}})$ parameters for the encoding and decoding when performing each update. Regarding the size of $I(t,l+1)$, we then associate to each $(t,l)$ a parameter cost for $\varepsilon$-approximation as follows.
%Following Theorem~\ref{theorem3},
%we associate to each $(t,l)$ a parameter cost for $\varepsilon$-approximation as follows.
\begin{enumerate}
    \item In max position aggregation,
    the parameter cost is $O(\varepsilon^{1-|I(t,l+1)|d/E_{l+1}})$.
    \item In global information aggregation,
    the parameter cost is $O(\varepsilon^{1-Td/E_{l+1}})$.
    \item In specific position aggregation,
    the parameter cost is $O(\varepsilon^{1-|I'(t,l+1)| \max_s |I(s,l)| d / E_{l+1}})$.
\end{enumerate}
Here $E_l$ denotes the total embedding dimension at layer $l$.
The total parameter cost of the \infoflow model is the sum of parameter costs over all $(t,l)$.

\begin{remark}
Feed-forward networks also incur a parameter cost of $O(\varepsilon^{-\gamma})$ for approximating nonlinear relations in target,
where $\gamma$ corresponds to the smoothness of the function.
We omit this term as it is not the dominant cost in regimes where Transformer approximation is fundamentally limited.
\end{remark}

\subsection{How \infoflow Works}
\label{sec:how-infoflow}

Given a Transformer structure and a target $F$,
\infoflow estimates the approximation rate of $F$ via the following two-step procedure.
The key quantity linking the complexity of the task to the capacity of the \infoflow model is the number of comparison,
which we define as follows.

{\paragraph{Number of Comparison of a target.}\label{sec:NumberOfComparison}
To quantify the computational complexity of a target function, we introduce the notion of the \emph{Number of Comparison}.

We consider target $F: \bigcup_{T}\mathcal{X}^T \to \mathbb{R}$ that takes inputs of arbitrary length, and denote by $F_T = \left. F\right|_{\mathcal{X}^T}$. Let $D_0(F)$, namely the maximum number of tokens to be retrieved, be the smallest integer such that for each $T$ and $X_T \,a.e.$ in Lebesgue measure, $|I_{F_T}(X_T)| \le D_0$. (We consider only those $F$ with finite $D_0(F)$). We first define the Tree of Comparison to aid the derivation of Number of Comparison. 
\begin{definition}
    A Tree of Comparison $\mathcal{T}$ is a full binary tree, with each leaf $V_{l}$ having an ordered index set $I_l$ allowing repetition of elements and depending only on $T$, and each internal node $V_{in}$ associated to a continuous function $f_{in}$. For each input $X_T$, the Tree of Comparison updates in the following way: For each internal node $V_{in}$ with children $V_{left}$ (with index set $I_{left}$)and $V_{right}$ (with $I_{right}$), $V_{in}$ is associated with index set $I = \argmax \{f_{in} (X[I_{left}]), f_{in} (X[I_{right}])\}$.
\end{definition}
By the structure of full binary tree, the Tree of Comparison is well-defined. We call $I_{final}(X_T, \mathcal{T})$ the index set associated to the root of $\mathcal{T}$ with input $X_T$, the number of leaves $N_{leaf}(\mathcal{T})=|\{V_l: V_l \in \mathcal{T}\}|$ the size of the Tree, and $\max_{V_l} |I_l|$ the dimension of the Tree. We then define the Dimension, Order and Number of Comparison of the target $F$:
\begin{definition}
    For $F$ with $D_0 = D_0(F) < \infty$. Let $\beta_1 \in \mathbb{Z}^+$ be the smallest of $\beta \in \mathbb{Z}^+$ that satisfies the following: For constant $C_0$ each $T$, there exists $D_1 \le D_0$ Trees of Comparison $\mathcal{T}_1, \dots, \mathcal{T}_{D_1}$, each with size $N_{leaf}(\mathcal{T}_i) \le C_0T^{\beta'}$ for $\beta' \le \beta$ and dimension no larger than $\beta$, such that $I_{F_T} (X_T) \subseteq \bigcup_{j=1}^{D_1} I_{final}(X_T, \mathcal{T}_j)$ for $X_T\,a.e.$ in Lebesgue measure. \\ We call $\beta_1$ the Dimension of Comparison of $F$, and the minimum of $\beta' \in \mathbb{Z}^+$ over all possible choices of such $\mathcal{T}_j$ with Dimension of Comparison $\beta_1$ is called the Order of Comparison of $F$. The minimum total number of pairwise comparison $N' = \sum_{j=1}^{D_1} (N_{leaf}(\mathcal{T}_j) -1)$ over all choices of Trees given $\beta = \beta_1$ is called the Number of Comparison of $F$.  
\end{definition}
We focus on $F$ with finite Dimension and Order of Comparison. Intuitively, the Number of Comparison of a target is how many tuples of tokens one needs to consider to compute the target output, and Dimension of Comparison is the maximum size of such tuples. For example, the generalized $D$-retrieval target 
$F(X_T) = F_0(\oplus_{i=1}^D \max_{1 \le t \le T} f_i(x(t)))$ studied in \cite{yu2026.EffectAttentionHead} has $\beta' = \beta_1 = 1$, with Number of Comparison upper bounded by $D(T-1)$ and lower bounded by $D(T-D)$.\\
However, in general, targets can involve comparison between larger tuples. A simple example is the target $F(X_T) = \min_{1 \le t_1,t_2,t_3 \le T} \|x(t_1) + x(t_2) + x(t_3)\|_2^2$, we can show that it has $\beta' = \beta_1 = 3$, with Number of Comparison upper bounded by $T^3-1$ and lower bounded by $\Omega(T^3)$. This example will be used later to show that Transformers are inherently limited in capturing such relationships in Section~\ref{sec:predictions}. We defer explanations of the two examples to Appendix~\ref{app:NumberOfComparison}.

}

\paragraph{Number of Comparison of \infoflow.}
When an \infoflow model approximates a target $F$ with Dimension of Comparison $\beta_1$,
comparisons can be performed in two ways.
First,
softmax attention performs at most $T-1$ comparisons per head.
Second,
within the subsequence $X[I(t,l)]$,
at most $|I(t,l)|^{\beta_1} - 1$ comparisons can be performed.
The total Number of Comparison of the \infoflow model is therefore
\begin{equation}
\sum_{t=1}^{T} \sum_{l=1}^{L-1} \Bigl(|I(t,l)|^{\beta_1} - 1 + \sum_{i=1}^{h_l}(T-1)\Bigr) + \sum_{l=1}^{L} \Bigl(|I(T+1,l)|^{\beta_1} - 1 + \sum_{i=1}^{h_l}(T-1)\Bigr).
\label{eq:num-comparison}
\end{equation}

\paragraph{Two-step estimation procedure.}
Given a Transformer structure $(L, h_1, \dots, h_L, E_1, \dots, E_L)$ and a target $F$,
the approximation rate is estimated as follows.
First,
the Number of Comparison of the \infoflow model must be no smaller than that of the target.
This yields a lower bound on $M = \max_{t,l} |I(t,l)|$,
which in turn gives a parameter cost lower bound of $O(\varepsilon^{1-Md/E_l})$ for $\varepsilon$-approximation of $F$.
Second,
one selects an update rule for each $(t,l)$ such that the resulting \infoflow model learns the target in the sense of active index set, %Section~\ref{sec: info flow},
yielding a corresponding upper bound on the parameter cost.
Together,
these two steps provide an estimate of the approximation rate of the Transformer on the target $F$ prior to actual training.

\section{Predictions and Experimental Validation}
\label{sec:predictions}

Before presenting new predictions,
we first verify that \infoflow is consistent with existing theoretical 
results from \cite{yu2026.EffectAttentionHead} and \cite{yun2020.AreTransformersUniversal}.

For the generalized $D$-retrieval target,
a single-layer Transformer with $s$ heads using only max position 
aggregation has at most $s(T-1)$ comparisons,
while the target requires at least $D(T-D)$ comparisons.
\infoflow therefore predicts that at least $s \ge D$ heads are needed 
for effective approximation,
which matches Theorem 2.2 of \cite{yu2026.EffectAttentionHead}.

Alternatively,
a single-layer one-head Transformer can use global information 
aggregation to $\varepsilon$-approximate the generalized $D$-retrieval 
target with $O(\varepsilon^{-Td/E})$ parameters.
This matches Theorem 2.3 of \cite{yu2026.EffectAttentionHead} and 
the universality results of \cite{yun2020.AreTransformersUniversal}.

Having verified consistency with known results,
we now turn to two cases  where direct 
theoretical analysis is currently intractable but \infoflow could provide correct predictions.

\subsection{Intrinsic Dimension Phenomenon in Two-Layer Transformers}

We consider the permutation-invariant target
\begin{equation}
    F(X_T) = \sum_{i=1}^D \max_{1 \le s,t \le T} x(s)^\top A_i x(t),
\label{eq:intrinsic-target}
\end{equation}
where $A_i \in \mathbb{R}^{d \times d}$ are distinct matrices.
We next apply \infoflow to analyze how effective transformers can approximate this target.

\paragraph{\infoflow prediction.}
Following the estimation in \infoflow, we have the following proposition. 

%We first establish the necessity of $h_1 = h_2 \ge D$ through the 
%following lemma.
\begin{Proposition}[Informal] \label{lemma:NOCofIntrinsic}
    The Number of Comparison of $F$ is upper bounded by $DT^2 + o(T^2)$; The Dimension and Order of Comparison of $F$ are $\beta_1 = \beta'=2$. \\
    A two-layer Transformer with $(h_1, h_2)$ heads using only max position aggregation has Number of Comparison at most $O(h_1 T^2)$, to learn $F$ we need $h_1 \ge D$. Also, there exists a two-layer Transformer with $(D,D)$ heads using only max 
    position aggregation that learns the target.  
\end{Proposition}
We defer the detailed explanation to Appendix~\ref{app: lemma1}. Together, the propositions show that $D$ is the intrinsic dimension of the target, in the sense that $h_1 = h_2 = D$ is the minimal head count for effective approximation.
\iffalse 
\begin{Proposition}[Informal]\label{lemma:intrinsic1}
For the target $F$ defined in \eqref{eq:intrinsic-target},
a two-layer Transformer with $(h_1, h_2)$ heads using only max 
position aggregation has number of comparison at most 
$O(h_1 T^2)$.
\end{Proposition}
\begin{Proposition}[Informal]\label{lemma:intrinsic2}
    There exists a two-layer Transformer with $(D,D)$ heads using only max 
    position aggregation that learns the target, namely $I_F(X_T) \subseteq I(T+1,2)$. 
\end{Proposition}\fi 
\paragraph{Experimental verification.}
We train two-layer Transformers with sequence length $T = 64$ on the target \eqref{eq:intrinsic-target} for $D = 2, 3, 4, 5, 6$,
comparing three head configurations for each $D$: $(D,D)$,
$(D-1,D)$,
and $(D,D-1)$.
As shown in Figure~\ref{fig:intrinsic},
the results exhibit a sharp phase transition at $h_1 = h_2 = D$.
The $(D,D)$ configuration achieves best validation NMSE on the order of $10^{-5}$ across all values of $D$,
while the $(D-1,D)$ and $(D,D-1)$ configurations achieve NMSE on the order of $10^{-3}$ to $10^{-2}$,
a gap of two to three orders of magnitude.
% This sharp separation across all tested values of $D$ provides strong empirical confirmation of the \infoflow prediction. Further training details are given in Appendix~\ref{app:exp:intrinsic}.
{
This sharp separation across all tested values of $D$ provides strong empirical confirmation of the \infoflow prediction.\\
Beyond validating the framework,
this result also reveals a structural insight about multi-layer Transformer approximation:
even in the multi-layer setting,
the number of heads remains a critical architectural parameter,
and must match the intrinsic dimension of the target for effective approximation.
This extends the single-layer finding of \cite{yu2026.EffectAttentionHead},
where insufficient head count leads to exponentially growing parameter cost,
to the multi-layer setting:
a two-layer Transformer with fewer than $D$ heads in the first layer 
cannot approximate a target of intrinsic dimension $D$ efficiently,
regardless of the second layer's configuration.
Further training details are given in Appendix~\ref{app:exp:intrinsic}.}

% \subsection{Approximation Hardness of Higher-Order Retrieval Tasks}

% \jht{cite the limitation paper, \cite{sanford}}

% We consider the triangle-center target
% \begin{equation}
%     F(X_T) = \min_{1 \le t_1,t_2,t_3 \le T} 
%     \|x(t_1) + x(t_2) + x(t_3)\|_2^2,
% \label{eq:triangle-target}
% \end{equation}
% whose number of comparison is $T^3 - 1$.
% \infoflow predicts that this target cannot be approximated effectively by any Transformer of fixed size as $T$ grows,
% regardless of the depth, number of heads or embedding dimension.
\subsection{Approximation Hardness of Higher-Order Retrieval Tasks}

The intrinsic dimension experiment demonstrates that when a target requires pairwise comparisons,
a two-layer Transformer with sufficiently many heads can approximate it efficiently.
However,
it is critical to observe that not all targets can be handled efficiently by any fixed Transformer architecture as $T$ grows.
The key insight from \infoflow is that when a target requires three or more simultaneous comparisons rather than pairwise ones,
the Number of Comparison of the target grows as a higher power of $T$,
and no fixed Transformer architecture can match this complexity without requiring number of parameters to grow exponentially with $T$.
This is consistent with the representational limitations of Transformers 
studied in \cite{sanford2023.RepresentationalStrengthsLimitations}.We then propose the the following proposition to demonstrate this phenomena. 
\begin{Proposition}[Informal]\label{lemma:triangle}
    If a target $F$ have Order of Comparison $\beta'>2$, then for an $L$-layer $H$-heads-per-layer Transformer with embedding dimension $E$, there exists a constant $C>0$ such that it requires a parameter cost of $\Omega(\varepsilon^{1-\frac{CT^{(\beta'-1)/\beta_1}}{LE}})$ for $\varepsilon$-approximation of $F_T$. 
\end{Proposition}
We defer the detailed explanation to Appendix~\ref{app:lemma2}. This proposition predicts that any target $F$ with Order of Comparison $\beta'>2$ cannot be approximated effectively by any transformer of fixed size as $T$ grows. 
To illustrate this phenomenon,
we consider the triangle-center target,
\begin{equation}\label{eq:triangle-target}
F(X_T) = \min_{1 \le t_1,t_2,t_3 \le T} 
\|x(t_1) + x(t_2) + x(t_3)\|_2^2,
\end{equation}
%which requires finding the triple of tokens whose sum is closest to the origin.
which have Order of Comparison $\beta'=3$ according to Appendix~\ref{app:Triangle_explain}.
Unlike the target in \eqref{eq:intrinsic-target},
which requires only pairwise comparisons between tokens and has number of comparison $O(T^2)$,
the triangle-center target requires three tokens to be compared simultaneously,
resulting in Number of Comparison $\Omega(T^3)$.
This cubic growth in comparison complexity is the source of the fundamental difficulty:
\infoflow predicts that this target cannot be approximated effectively by any Transformer of fixed size as $T$ grows,
regardless of the depth,
number of heads,
or embedding dimension.
\paragraph{\infoflow prediction.}
Following Proposition~\ref{lemma:triangle}, for the target $F$ defined in \eqref{eq:triangle-target}, it requires an $L$-layer $H$-heads-per-layer Transformer with embedding dimension $E$ a parameter cost of $O(\varepsilon^{1-\frac{T^{2/3}}{6LE}})$ for $\varepsilon$-approximation of $F$.
\iffalse
\begin{Proposition}[Informal]\label{lemma:triangle}
For the target $F$ defined in \eqref{eq:triangle-target},

\end{Proposition}
\fi
%We defer the detailed explanation to Appendix~\ref{app:lemma2}.
Since the exponent $\frac{T^{2/3}}{LE}$ grows with $T$, for any fixed $L$ and $E$,
the parameter cost of approximating $F$  must be super-polynomial in $1/\varepsilon$ for large $T$,
regardless of the number of heads or embedding dimension. This demonstrates a fundamental limitation that cannot be overcome by increasing model size.

\paragraph{Experimental verification.}
We train two-layer Transformers of three different sizes on the target 
\eqref{eq:triangle-target} for sequence lengths $T \in \{3, 4, 5, 6, 8, 10, 12, 16, 24, 32, 48, 64\}$.
We consider the following head configurations:
$(h_1,h_2) = (2,2)$ with embedding dimension $24$,
$(h_1,h_2) = (4,4)$ with embedding dimension $24$,
and $(h_1,h_2) = (4,4)$ with embedding dimension $48$.
As shown in Figure~\ref{fig:triangle},
all three configurations exhibit rapidly growing NMSE as $T$ increases,
with NMSE exceeding $0.85$ at $T = 16$ and approaching $1.0$ for $T \ge 32$,
regardless of model size.
The three configurations perform almost identically across all values 
of $T$,
indicating that increasing the number of heads or the embedding dimension does not improve approximation within the tested range.
These results are consistent with the \infoflow prediction that the parameter cost of approximating $F$ grows with $T$ for any fixed architecture. Further training details are given in Appendix~\ref{app:exp:triangle}.

\begin{figure}[h]
\centering
\begin{subfigure}[t]{0.48\linewidth}
\centering
\includegraphics[width=\linewidth]{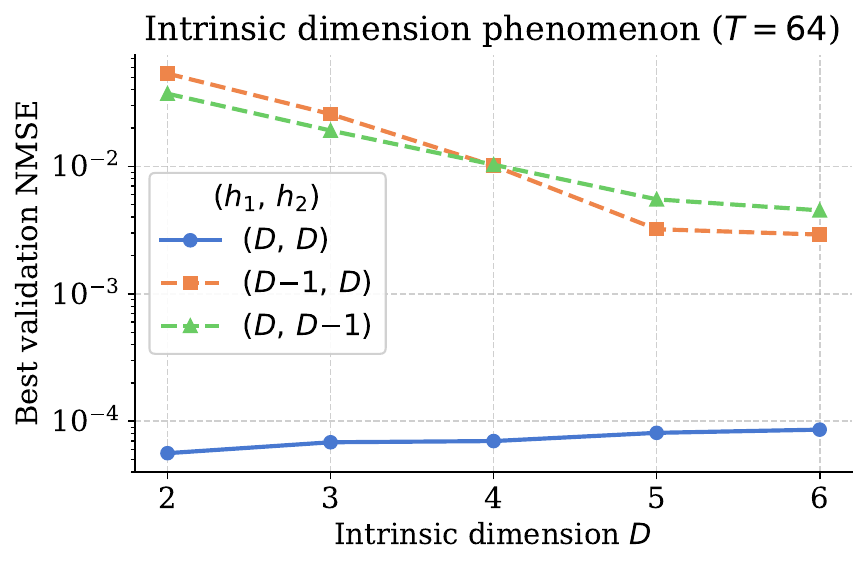}
\caption{Best validation NMSE for two-layer Transformers trained on the target 
\eqref{eq:intrinsic-target},
showing a sharp phase transition at $h_1 = h_2 = D$ that confirms that $D$ is the intrinsic dimension of the target.}
\label{fig:intrinsic}
\end{subfigure}
\hfill
\begin{subfigure}[t]{0.5\linewidth}
\centering
\includegraphics[width=\linewidth]{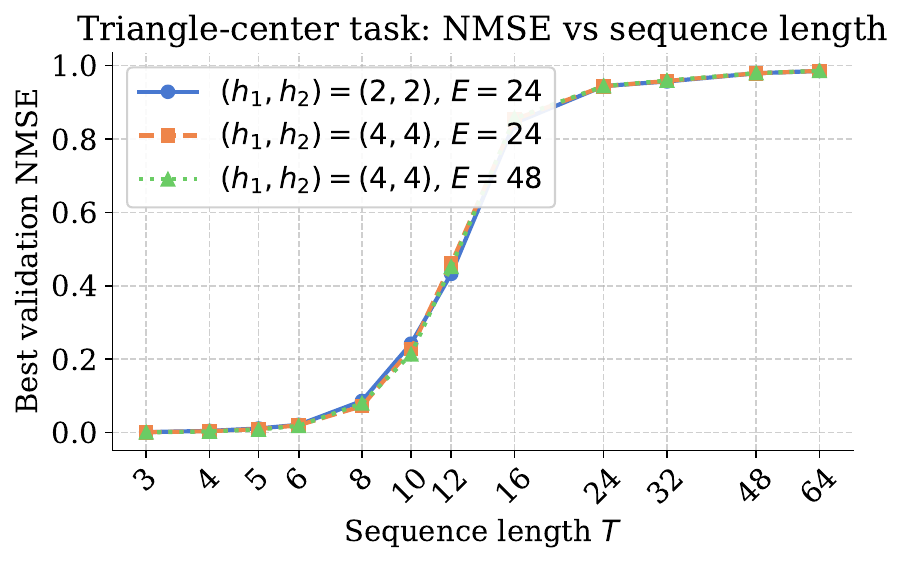}
\caption{Best validation NMSE against sequence length $T$ for three Transformer configurations trained on \eqref{eq:triangle-target},
consistent with the \infoflow prediction that approximation difficulty grows with $T$ regardless of model size.}
\label{fig:triangle}
\end{subfigure}
\caption{
Experimental validation of the intrinsic dimension phenomenon (left) 
and the approximation hardness of the triangle-center task (right).}
\label{fig:predictions}
\end{figure}

\section{Discussion and Outlook}

In this work,
we proposed \infoflow,
a coarse-grained framework for analyzing the approximation properties of multi-layer Transformers.
Building on a depth separation result and three structural mechanisms of information propagation,
\infoflow associates to each token position and layer an information set together with a parameter cost law,
yielding a tractable abstraction that recovers known approximation bounds and yields concrete predictions where direct theoretical analysis is intractable.
Experimental validation on the intrinsic dimension phenomenon and the triangle-center task confirms the predictions of \infoflow,
demonstrating its utility as a tool for reasoning about approximation efficiency prior to training.
Several directions remain open for future investigation.
On the theoretical side,
a more complete characterization of information propagation modes would strengthen the foundation of \infoflow,
and extending it to practical Transformer variants is an important open problem.
On the practical side,
\infoflow can potentially guide architecture design by predicting which tasks are well-suited to a given Transformer structure and informing hyperparameter choices such as head count and embedding dimension.

\paragraph{Limitation}
The parameter cost law in \infoflow does not account for the role of attention weight matrix rank in the attention score computation,
a limitation also present in \cite{yu2026.EffectAttentionHead} and studied by \cite{amsel2024.QualityQuantityAttention}.
Incorporating rank into the cost law is a natural direction for future work.
Additionally,
the three propagation modes are motivated by synthetic tasks,
and their relevance to real-world Transformer behavior requires further validation.
Finally,
the current framework applies to retrieval-type targets where the active index set does not grow with $T$,
and extending it to more general targets would require new theoretical tools.

\newpage
\bibliographystyle{plainnat}
\bibliography{neurips_2026}

\newpage
\appendix
\crefalias{section}{appendix}

\section{Formal statements and Proofs of Theorems }
\subsection{Formal statement and proof of Theorem~\ref{theorem1}}\label{app:thm1}
We choose $\mathcal{X} = [-1,1]^{3}$ and consider the following task: the input $X_T \in \mathcal{X}^T$, the target $F:\mathcal{X}^T \to \mathbb{R}$ is defined as 
\begin{equation}
    F(X_T) = \min_{1 \le s, t \le T} 2(1 + x(s)^T x(t))
\end{equation}
For a model $M: \mathcal{X}^T \to \mathbb{R}$, define 
\begin{equation}
    \delta(M) = \sup_{X_T \in \mathcal{X}^T} |F(X_T) - M(X_T)|
\end{equation}
For the theorem to hold, we also post the following constraints on the models. 
\begin{itemize}
    \item The embedding $P_\phi$ satisfies  
\[
    \|P_\phi(x(s),s)\|_2 \leq 1, \qquad \forall\, s \in [T], \; X_T \in \mathcal{X}^T,
\]
ensuring embedded inputs remain uniformly bounded.
    \item All weights in all feed-forward blocks $ F_l$ and $\hat{F}$ and entries of the attention matrices $\{W_{Q,i,l},W_{K,i,l},W_{V,i,l}\},W_{O,l}$ are bounded in magnitude by $1$, ensuring stability of the model.
\end{itemize}

Under these assumptions, we then propose the formal theorem showing that this task $F$ is easy for two-layer transformer and hard for single layer transformer. 
\begin{theorem}[Formal Version of Theorem~\ref{theorem1}]
    The following two cases holds. 
    \begin{itemize}
        \item Let $M_1: \mathcal{X}^T \to \mathbb{R}$ be a single-layer transformer with $H$ heads and per-head embedding dimension $n$ defined above. If $M_1$ satisfies that $\delta(M_1) \le \varepsilon$, then $M_1$ have parameter count of the order $\Omega(\varepsilon^{-k})$, where $k = \frac{T - H - 1}{H(n+1) +1} - 1$. 
        \item $\forall \varepsilon>0$, there exists a two-layer transformer $M_2$ with $1$ head in both layers and per-head dimension $6$ defined above, with parameter count at most of order $O(\varepsilon^{-1})$ such that $\delta(M_2) \le \varepsilon$
    \end{itemize}
\end{theorem}
\paragraph{Proof Sketch} The proof of the first part of Theorem\ref{theorem1} follows the methods and notations of Theorem 2.2 in \citep{yu2026.EffectAttentionHead}, by using the Pigeonhole Theorem to construct two subsequences  which are indistinguishably close after softmax attention, on the order of $O(\varepsilon^{k+1})$, while their target values differs by $4\varepsilon$. Then we extend them to full sequences to complete the proof. 

The second part of Theorem~\ref{theorem1} is proven with explicit construction. 
\begin{proof}
\textbf{First Part of Theorem~\ref{theorem1}}
We consider only the nontrivial case, where $T_0 = T-H-1 > 0$. \\
We use the following notations (let $E = nH$):
\paragraph{Notations}
For each head $i=1,\dots,s$, define the attention weight function  
\[
    \lambda_i(x,t) \;=\; 
    \exp\!\bigl(\beta \, (W_{K,i,1} P_\phi(x,t))^\top W_{Q,i,1}c_0 \bigr),\quad  i=1, \dots, H
\]
and the value mapping  
\[
    v_i(x,t) \;=\; W_{V,i,1} P_\phi(x,t) \;\in\; \mathbb{R}^n, \quad i=1, \dots, H
\]
where $\gamma>0$ is the softmax scaling factor, 
$W_{Q,i,1},W_{K,i,1}\in\mathbb{R}^{n\times E}$ are the query and key projections, 
and $W_{V,i,1}\in\mathbb{R}^{n\times E}$ is the value projection for head $i$.\\

We denote by $Q_0$ the following set
\begin{equation}
    Q_0 = \{x: x\in \mathbb{S}^2 \text{ with } x_1 \in [-1/2, 0], x_2 \in [-1/2, 0], x_3 < 0\}
\end{equation}
We choose $H+1$ points $b_1, \dots, b_{H+1} \in Q_0$ satisfying the followings:\\
(1)$\|b_i - b_j\|_2 \ge 4L$,  $\forall i \ne j$, \\
(2)$B(b_i, L) \cap \mathbb{S}^2 \subset Q_0$.\\
It is clear that there exists $L = \frac{C}{H}$ for positive absolute constant $C$ for this to hold. We choose such $L$ and such $b_1, \dots, b_{H+1}$. 

Let $K_i = - B(b_i, L) \cap Q_0$, $i=1, \dots, H+1$, and let $P_0 = \bigcup_{i=1}^{H+1} K_i$, we recursively define the following pairs $(y_i, t_i)$ for $i=1, \dots, H$ as follows. 
\begin{itemize}
    \item For the first head,
    \[
        (y_1,t_1) \;=\; \arg\max_{\substack{y \in P_0 \\ t \in  [T]}}
        \,\lambda_1(y,t).
    \]
    \item For $j>1$,
    \[
        (y_j,t_j) \;=\; \arg\max_{\substack{y \in P_0 \\ t \in [T] \\ t \notin \{t_1,\dots,t_{j-1}\}}}
        \,\lambda_j(y,t).
    \]
    \item If maximum can be obtained at multiple $(y, t)$, then choose one of them. 
\end{itemize}
Let $Y = \{y_1,\dots,y_s\}$ (Consider $Y$ as a "set" allowing repetition).  
Since the sets $K_1,\dots,K_{H+1}$ are pairwise disjoint and $H<H+1$, 
there exists at least one index $i \in \{1,\dots,H+1\}$ such that
\[
    K_i \,\cap\, Y \;=\; \varnothing.
\]
Without loss of generality, we assume that $i=1$, and $t_j = T-j$ for $j=1, \dots, H$. 

\paragraph{Sequences to be considered}
We then consider constructing the two sequences that are close after attention layer. 

We denote by $M = T_0\lfloor \frac{L^2}{4T_0\varepsilon} \rfloor$. And for $i =1, \dots, M$, we construct points $p_i$ satisfying the following:
\begin{equation}
    p_i \in \{x: x \in K_1, \|x - (-b_1)\|_2^2 = \frac{(i-1)L^2}{M}\}
\end{equation}
As it is clear that $\{x: x \in K_1, \|x - (-b_1)\|_2^2 = \frac{(i-1)L^2}{M}\}$ must be non-empty, we can pick such $p_i$ for $i=1, \dots, M$. 

With $\{p_i\}$ constructed, we construct sets $U_j$ for $j=1, \dots, T_0$:
\begin{equation}
    U_j = \{p_i : \frac{(j-1)M}{T_0}< i \le  \frac{jM}{T_0}\}
\end{equation}
We have that $|U_j| = \frac{M}{T_0} = \lfloor \frac{L^2}{4T_0\varepsilon} \rfloor = O(1/\varepsilon)$. 

\paragraph{Existence of Two Distinct Sequences} We then claim that there exists two distinct sequences $Z_{1, T} = z_1(1), \dots ,z_1(T)$ and $Z_{2, T} = z_2(1), \dots, z_2(T)$ satisfying the following: 
\begin{itemize}\label{z1andz2}
    \item $z_1(T) = z_2(T) = b_1$. 
    \item $z_1(T-i) = z_2(T-i) = y_i$, for $i=1, \dots, H$.
    \item $z_1(t), z_2(t) \in U_t$, for $t=1, \dots, T_0$.
    \item $\left\|\frac{\sum_{t=1}^{T}\lambda_i(z_1(t), t)v_i(z_1(t),t) }{\sum_{t=1}^{T}\lambda_i(z_1(t), t)}
    - \frac{\sum_{t=1}^{T}\lambda_i(z_2(t), t)v_i(z_2(t),t) }{\sum_{t=1}^{T}\lambda_i(z_2(t), t)}\right\|_2
    \le \frac{\varepsilon^{k+1}}{6T}$, for $i=1,\dots,H$.
    \item $\frac{\sum_{t=1}^{T}\lambda_i(z_1(t), t)}{\sum_{t=1}^{T}\lambda_i(z_2(t), t)} 
        \in \left[1/ (1+\tfrac{\varepsilon^{k+1}}{24T^2}), \, 1+\tfrac{\varepsilon^{k+1}}{24T^2}\right]$.
\end{itemize}
We then justify the existence of such two distinct sequence by comparing the order of $1/\varepsilon$ on both side of the conditions. 
\begin{itemize}
    \item We first consider the order on the choice space of $Z_{1,T}, Z_{2,T}$. As each $U_t$ have size $O(1/\varepsilon)$, $t=1, \dots, T_0$, Then the size of the choice space is $O(1/\varepsilon^{T_0})$. 
    \item We then consider the order on the constraint space. Noticing that $\left\|\frac{\sum_{t=1}^{T}\lambda_i(z_j(t), t)v_i(z_j(t),t) }{\sum_{t=1}^{T}\lambda_i(z_j(t), t)}\right\|_2 \le 1$ as it is a convex combination of vectors of norm no greater than $1$, $j=1,2$. Thus discretizations to achieve accuracy $O(\varepsilon^{k+1})$ would lead to $O(1/\varepsilon^{(k+1)nH})$ possibilities across $H$ heads. Let $M_i= \max \{\max_{1\le j \le H}\lambda_i(y_j, t_j), \lambda_i(b_1, T) \}$ for $i=1, \dots, H$, we have that $M_i\le \sum_{t=1}^{T}\lambda_i(z_j(t), t) \le TM_i$, $j=1,2$. Therefore the discretizations to achieve accuracy $O(\varepsilon^{k+1})$ of the term $\frac{\sum_{t=1}^{T}\lambda_i(z_1(t), t)}{\sum_{t=1}^{T}\lambda_i(z_2(t), t)}$ leads to $O(1/\varepsilon^{(k+1)H})$ possibilities across $H$ heads. Thus it reaches to a total of $O(1/\varepsilon^{(k+1)(n+1)H})$ possibilities. 
\end{itemize}
As $T_0 = (k+1)(H(n+1)+1) > (k+1)H(n+1)$, we have 
\begin{equation}
    O(\frac{1}{\varepsilon^{T_0}}) \gg O(\frac{1}{\varepsilon^{(k+1)(n+1)H}})
\end{equation}
Therefore, by the pigeonhole principle, there exists two distinct sequences $Z_{1, T}$ and $Z_{2, T}$ satisfying all the conditions in~\ref{z1andz2}.

Consider index set $J = \{t : z_1(t) \ne z_2(t)\}$. By $Z_{1, T} \ne Z_{2, T}$, we have that $J \ne \varnothing$, and by construction of $Z_{1,T}, Z_{2, T}$, we have that $J \subseteq \{1, \dots, T_0\}$. 
We then construct two sequences $W_{1,T}, W_{2,T}$ from $Z_{1, T}, Z_{2, T}$. 
\begin{itemize}
    \item For $t \in J$, $w_1(t) = z_1(t)$, $w_2(t) = z_2(t)$. 
    \item For $t \in \{1, \dots, T_0\} - J$, $w_1(t) = w_2(t) = y_1$.
    \item For $T-H\le t \le T$, $w_1(t) = w_2(t) = z_1(t)$. (Notice that for $T-H\le t \le T$, $z_1(t) = z_2(t)$. )
\end{itemize}
Denote by $I' = \{1, \dots, T_0\} - J$, $K' = \{T-H, \dots, T\}$, and for $i=1, \dots, H$, We define the following notations for simplicity of calculation, as in \cite{yu2026.EffectAttentionHead}.
\begin{itemize}
    \item $Q_{1,i} = \sum_{t \in J} \lambda_i(w_1(t), t)$.
    \item $Q_{2,i}= \sum_{t \in J} \lambda_i(w_2(t), t)$.
    \item $V_{1,i} = (\sum_{t \in J} \lambda_i(w_1(t), t) v_i(w_1(t), t))/Q_{1,i}$.
    \item $V_{2,i} = (\sum_{t \in J} \lambda_i(w_2(t), t) v_i(w_2(t), t))/Q_{2,i}$.
    
    \item $Q_{3,i} = \sum_{t \in I'} \lambda_i(z_1(t), t)$, which is of the same value if defined on $Z_{2, T}$.
    \item $V_{3,i} = (\sum_{t \in I'} \lambda_i(z_1(t), t) v_i(z_1(t), t))/Q_{3,i}$, which is of the same value if defined on $Z_{2, T}$.
    \item $Q_{4, i} = \sum_{t \in I'} \lambda_i(w_1(t), t)$, which is of the same value if defined on $W_{2,T}$.
    \item $V_{4, i} = (\sum_{t \in I'} \lambda_i(w_1(t), t) v_i(w_1(t), t))/Q_{4,i}$, which is of the same value if defined on $W_{2, T}$.
    \item $Q_{5,i} = \sum_{t \in K'} \lambda_i(w_1(t), t)$, which is of the same value on all $Z, W$. 
    \item $V_{5, i} = (\sum_{t \in K'} \lambda_i(w_1(t), t) v_i(w_1(t), t))/Q_{5,i}$, which is of the same value on all $Z,W$. 
\end{itemize}

We then have the following equations:
\begin{itemize}
    \item $\sum_{t=1}^{T}\lambda_i(z_1(t), t) = Q_{1, i} + Q_{3, i} + Q_{5, i}$.
    \item $\sum_{t=1}^{T}\lambda_i(z_2(t), t) = Q_{2, i} + Q_{3, i} + Q_{5, i}$.
    \item $\sum_{t=1}^{T}\lambda_i(w_1(t), t) = Q_{1, i} + Q_{4, i} + Q_{5, i}$.
    \item $\sum_{t=1}^{T}\lambda_i(z_2(t), t) = Q_{2, i} + Q_{4, i} + Q_{5, i}$.
    \item $\frac{\sum_{t=1}^{T}\lambda_i(z_1(t), t)v_i(z_1(t),t) }{\sum_{t=1}^{T}\lambda_i(z_1(t), t)} = \frac{Q_{1,i}V_{1, i} + Q_{3,i}V_{3, i}+ Q_{5,i}V_{5, i}}{Q_{1, i} + Q_{3, i} + Q_{5, i}}$.
    \item $\frac{\sum_{t=1}^{T}\lambda_i(z_2(t), t)v_i(z_2(t),t) }{\sum_{t=1}^{T}\lambda_i(z_2(t), t)} = \frac{Q_{2,i}V_{2, i} + Q_{3,i}V_{3, i}+ Q_{5,i}V_{5, i}}{Q_{2, i} + Q_{3, i} + Q_{5, i}}$.
    \item $\frac{\sum_{t=1}^{T}\lambda_i(w_1(t), t)v_i(w_1(t),t) }{\sum_{t=1}^{T}\lambda_i(w_1(t), t)} = \frac{Q_{1,i}V_{1, i} + Q_{4,i}V_{4, i}+ Q_{5,i}V_{5, i}}{Q_{1, i} + Q_{4, i} + Q_{5, i}}$.
    \item $\frac{\sum_{t=1}^{T}\lambda_i(w_2(t), t)v_i(w_2(t),t) }{\sum_{t=1}^{T}\lambda_i(w_2(t), t)} = \frac{Q_{2,i}V_{2, i} + Q_{4,i}V_{4, i}+ Q_{5,i}V_{5, i}}{Q_{1, i} + Q_{4, i} + Q_{5, i}}$.
\end{itemize}

Without loss of generality, we assume $Q_{1, i} \ge Q_{2,i}$.
By calculation, we have 
\begin{align}
    &\frac{Q_{1,i}V_{1,i} + Q_{4,i}V_{4,i}+ Q_{5,i}V_{5,i}}{Q_{1,i}+ Q_{4,i} + Q_{5,i}}- \frac{Q_{2,i}V_{2,i} + Q_{4,i}V_{4,i} + Q_{5,i}V_{5,i}}{Q_{2,i}+ Q_{4,i} + Q_{5, i}}\\
    =& \frac{(Q_{2,i} - Q_{1,i})(Q_{4, i}(V_{4,i} - V_{2,i}) + Q_{5,i}(V_{5,i} - V_{2,i}))}{(Q_{1,i}+Q_{4,i}+Q_{5,i})(Q_{2,i}+Q_{4,i} + Q_{5, i}) } + \frac{Q_{1,i}}{Q_{1,i}+Q_{4,i} + Q_{5,i}}(V_{1,i} - V_{2,i}).
\end{align}
As $Q_{5,i} \ge M_i \ge \frac{Q_{1,i}+Q_{3,i}+Q_{5,i}}{T}$, we have that 
\begin{align}\label{ineq:21}
    &\|\frac{(Q_{2,i} - Q_{1,i})(Q_{4, i}(V_{4,i} - V_{2,i}) + Q_{5,i}(V_{5,i} - V_{2,i}))}{(Q_{1,i}+Q_{4,i}+Q_{5,i})(Q_{2,i}+Q_{4,i} + Q_{5, i}) }\| \\
    \le & \|\frac{(Q_{2,i} - Q_{1,i})(2Q_{4, i} + 2Q_{5,i})}{(Q_{1,i}+Q_{4,i}+Q_{5,i})(Q_{2,i}+Q_{4,i} + Q_{5, i})}\| \\
    \le & \|\frac{2(Q_{2,i} - Q_{1,i})}{(Q_{1,i}+Q_{4,i}+Q_{5,i})}\| \\
    \le & \|\frac{2T(Q_{2,i} - Q_{1,i})}{(Q_{1,i}+Q_{3,i}+Q_{5,i})}\| \\
    = & \|1- \frac{\sum_{t=1}^{T}\lambda_i(z_2(t), t)}{\sum_{t=1}^{T}\lambda_i(z_1(t), t)}\|\\
    \le & \|\frac{T \varepsilon^{k+1}}{12T^2 }\|\\
    \le & \frac{\varepsilon^{k+1}}{12T}
\end{align}
Similarly, we also have
\begin{equation}
    \|\frac{(Q_{2,i} - Q_{1,i})(Q_{3, i}(V_{3,i} - V_{2,i}) + Q_{5,i}(V_{5,i} - V_{2,i}))}{(Q_{1,i}+Q_{3,i}+Q_{5,i})(Q_{2,i}+Q_{3,i} + Q_{5, i}) }\| \le \frac{\varepsilon^{k+1}}{12T}
\end{equation}
Then we have 
\begin{align}
    &\|\frac{Q_{1,i}}{Q_{1,i}+Q_{3,i} + Q_{5,i}}(V_{1,i} - V_{2,i})\|\\
    \le &\|\frac{Q_{1,i}V_{1,i} + Q_{3,i}V_{3,i}+ Q_{5,i}V_{5,i}}{Q_{1,i}+ Q_{3,i} + Q_{5,i}}- \frac{Q_{2,i}V_{2,i} + Q_{3,i}V_{3,i} + Q_{5,i}V_{5,i}}{Q_{2,i}+ Q_{3,i} + Q_{5, i}}\|\\
    +& \|\frac{(Q_{2,i} - Q_{1,i})(Q_{3, i}(V_{3,i} - V_{2,i}) + Q_{5,i}(V_{5,i} - V_{2,i}))}{(Q_{1,i}+Q_{3,i}+Q_{5,i})(Q_{2,i}+Q_{3,i} + Q_{5, i}) } \|\\
    = & \|\frac{\sum_{t=1}^{T}\lambda_i(z_1(t), t)v_i(z_1(t),t) }{\sum_{t=1}^{T}\lambda_i(z_1(t), t)} - \frac{\sum_{t=1}^{T}\lambda_i(z_2(t), t)v_i(z_2(t),t) }{\sum_{t=1}^{T}\lambda_i(z_2(t), t)} \| \\
    +& \|\frac{(Q_{2,i} - Q_{1,i})(Q_{3, i}(V_{3,i} - V_{2,i}) + Q_{5,i}(V_{5,i} - V_{2,i}))}{(Q_{1,i}+Q_{3,i}+Q_{5,i})(Q_{2,i}+Q_{3,i} + Q_{5, i}) } \|\\
    \le & \frac{\varepsilon^{k+1}}{6T} + \frac{\varepsilon^{k+1}}{12T}\\
    =& \frac{\varepsilon^{k+1}}{4T}
\end{align}
Then we have 
\begin{align}
    &\|\frac{Q_{1,i}}{Q_{1,i}+Q_{4,i} + Q_{5,i}}(V_{1,i} - V_{2,i})\|\\
    \le & \|\frac{TQ_{1,i}}{Q_{1,i}+Q_{3,i} + Q_{5,i}}(V_{1,i} - V_{2,i})\|\\
    \le & \frac{\varepsilon^{k+1}}{4}
\end{align}
Thus 
\begin{align}
    & \|\frac{\sum_{t=1}^{T}\lambda_i(w_1(t), t)v_i(w_1(t),t) }{\sum_{t=1}^{T}\lambda_i(w_1(t), t)} - \frac{\sum_{t=1}^{T}\lambda_i(w_2(t), t)v_i(w_2(t),t) }{\sum_{t=1}^{T}\lambda_i(w_2(t), t)} \|\\
    =&\|\frac{Q_{1,i}V_{1,i} + Q_{4,i}V_{4,i}+ Q_{5,i}V_{5,i}}{Q_{1,i}+ Q_{4,i} + Q_{5,i}}- \frac{Q_{2,i}V_{2,i} + Q_{4,i}V_{4,i} + Q_{5,i}V_{5,i}}{Q_{2,i}+ Q_{4,i} + Q_{5, i}}\|\\
    \le& \|\frac{Q_{1,i}}{Q_{1,i}+Q_{4,i} + Q_{5,i}}(V_{1,i} - V_{2,i})\| +  \|\frac{(Q_{2,i} - Q_{1,i})(Q_{3, i}(V_{3,i} - V_{2,i}) + Q_{5,i}(V_{5,i} - V_{2,i}))}{(Q_{1,i}+Q_{3,i}+Q_{5,i})(Q_{2,i}+Q_{3,i} + Q_{5, i}) }\| \\
    \le& \frac{\varepsilon^{k+1}}{4} + \frac{\varepsilon^{k+1}}{12T} \\
    \le& \frac{\varepsilon^{k+1}}{2}
\end{align}
This indicated that $W_{1,T}, W_{2,T}$ have very close representations. 

Now we show that $F(W_{1,T})$ and $F(W_{2,T})$ are at least separated by $\varepsilon$. 

It is clear that $F(X_T) = \min_{1 \le s, t \le T}\|x(s) + x(t)\|_2^2$ if the entire $X_T$ lies on $\mathbb{S}^{2}$, which both $W_{1, T}, W_{2, T}$ satisfies. 
As $w_1(T) = w_2(T) = b_1$ is the only point on each sequence in $Q_0$, and the rest of the sequence are in $-Q_0$, it is clear that $F(W_{i,t}) = \min_{1 \le s \le T-1} \|w_i(t) - (-b_1)\|_2^2$. Furthermore, by the fact that $Y \cap B(-b_1, L) = \varnothing$, this further reduces to $F(W_{i,t}) = \min_{s \in J} \|w_i(t) - (-b_1)\|_2^2$. 

By the construction of $U_t$ and $Z_{1, T}, Z_{2, T}$, let $j^* = \min_{t \in J}$, we have that $F(W_{i,t}) = \|w_i(j^*) - (-b_1)\|_2^2 =  \|z_i(j^*) - (-b_1)\|_2^2$. From $z_1(j^*) \ne z_2(j^*)$, we have that $|\|z_1(j^*) - (-b_1)\|_2^2 - \|z_2(j^*) - (-b_1)\|_2^2| \ge 4\varepsilon$. From $\delta(M_1) \le \varepsilon$, we have that $|M_1(W_{1, T}) - M_1(W_{1, T})| \le 2 \varepsilon$.  

Then by Lemma 4 in \cite{yu2026.EffectAttentionHead}, which states that for two vectors $v_1,v_2 \in \mathbb{R}^n$. If  
\[
    \|v_1 - v_2\|_2 \le A
    \quad\text{and}\quad
    \|\hat{F}(v_1) - \hat{F}(v_2)\| \ge B ,
\]
where $\hat{F}:\mathbb{R}^n \to \mathbb{R}^m$ is a two-layer feed-forward network 
satisfying with Tanh-activation and bounded weights, then $\hat{F}$ must use at least
\[
    \Omega\!\left(\frac{B}{A\sqrt{n}}\right)
\]
parameters. Having $A  \le \frac{\varepsilon^{k+1}}{2}$, and $B\ge \varepsilon$

Thus the parameter count is lower bounded by $\Omega(\frac{1}{\varepsilon^k})$. 

\textbf{Second Part of Theorem~\ref{theorem1}}
We directly construct this two-layer transformer. 
\begin{enumerate}
    \item $x_1(t)=\hat{x}(t) = P_\phi(x(t), t) = (\frac{x(t)}{3},0) \in \mathbb{R}^6$. This embedding requires only $O(1)$ parameters. 
    \item $W_{Q,1,1}^T W_{K,1,1} = \begin{pmatrix}
        - I_3 & 0 \\
        0 & 0
        \end{pmatrix} \in \mathbb{R}^{6 \times6}$. We then have that $(W_{Q,1,1}x_1(s))^T W_{K,1,1}x_1(t) = - \frac{1}{9}x(s)^Tx(t)$. 
    \item $W_{V,1,1} = I_6$, $W_{O,1} = \begin{pmatrix}
        0 & 0 \\
        I_3 & 0
        \end{pmatrix}$. We then have $x'_1(t) = (\frac{1}{3}x(t), v(t)) \in \mathbb{R}^6$, where $v(t)$ is the actual vector obtained from softmax attention. By the fact that log attention score $(W_{Q,1,1}x_1(s))^T W_{K,1,1}x_1(t) = - \frac{1}{9}x(s)^Tx(t)$ and we can apply a scaling factor $\beta >0$ to this term in soft attention, we have that when $\beta \to \infty$, $x(t)^T v(t) \to \frac{1}{9} \min_{1 \le s \le T} x(s)^Tx(t)$ uniformly for all $X_T \in \mathcal{X}^T$. We set $\beta$ to be large enough such that $|x(t)^T v(t) - \frac{1}{9} \min_{1 \le s \le T} x(s)^Tx(t)|\le \frac{\varepsilon}{54}$.
    \item Denote by $G:[-1,1]^6 \to [-1,1]^6$ the following function: $G((x, y)) = (x^Ty,\frac{1}{2},0,0,0,0)$, where $(x,y) \in \mathbb{R}^6, x,y \in \mathbb{R}^3$. Then there exists a two-layer Tanh-activated FFN $F_1$ with at most $O(1/\varepsilon^1)$ parameters such that $|F_1((x,y)) - G((x,y))| \le \frac{\varepsilon}{54}$. (As we can rewrite $x_iy_i = \frac{1}{4} [(x_i+y_i)^2 - (x_i-y_i)^2]$, this approximation of inner product is transformed to one dimensional FFN approximation problem, therefore $O(1/\varepsilon^1)$ approximation rate can be achieved by FFNs.)
    \item Write $x_2(t) = (a(t), \frac{1}{2}, 0,0,0,0) \in \mathbb{R}^6$. Let $W_{Q,1,2}^T W_{K,1,2} = A \in \mathbb{R}^{6 \times 6}$, where $A_{2,1} = - 2$ is the only non-zero element of $A$. Thus we have that $(W_{Q,1,2}c_2)^T W_{K,1,1}x_2(t) = - a(t)$. 
    \item $W_{V,1,2} = I_6$, $W_{O,2} = \begin{pmatrix}
        0 & 0 \\
        I_3 & 0
        \end{pmatrix}$. With the scaling factor $\beta >0$, we also have a large enough $\beta$ such that the $4$-th element of $c_2'$ satisfies that $|(c_2')_{4} - \min_{1 \le t \le T}a(t)| \le \frac{\varepsilon}{54}$. Then the output feed-forward block maps $c_2'$ to $2+18(c_2')_{4}$ with $O(1)$ parameters. 
    \item The total error is 
    \begin{align}
        & |M_2(X_T) - F(X_T)|\\
        \le & 18 [\max_{1 \le t \le T}(|x(t)^T v(t) - \frac{1}{9} \min_{1 \le s \le T} x(s)^Tx(t)| + |a(t) -\frac{1}{9} \min_{1 \le s \le T} x(s)^Tx(t)|) \\
        +& |(c_2')_{4} - \min_{1 \le t \le T}a(t)|]\\
        \le & \frac{\varepsilon}{3}+\frac{\varepsilon}{3}+\frac{\varepsilon}{3} \\
        \le & \varepsilon
    \end{align}
\end{enumerate}
By this construction we have proven the second part of Theorem~\ref{theorem1}.

\end{proof}

\begin{remark}
    If we choose $\mathcal{X} = [0,1]^3$, we can also choose a similar target for Theorem~\ref{theorem1} to hold. 
\end{remark}

\subsection{Formal statement and proof of Theorem~\ref{theorem2}}\label{app:thm2}
It suffices to show this statement in the scalar case with $\mathcal{X} = [0,1]$ ($X_T \in [0,1]^T$). For a given $X_T$, there exists a permutation $\sigma$ such $x(\sigma(1)) \ge x(\sigma(2)) \ge \dots, \ge x(\sigma(T))$. Then for a fixed $k$, define $F_k: [0,1]^T \to [0,1]$ to be $F_k(X_T) = x(\sigma(k))$. 

We have the following theorem showing that softmax attention cannot retrieval the $k$-th largest value efficiently. 
\begin{theorem}[Formal Version of Theorem~\ref{theorem2}]
    Let $f_1: [0,1] \to [0,1]^n$ and $f_2: [0,1]^n \to [0,1]$ be two-layer Tanh-activated Feed forward Neural Networks, with weights bounded by $1$. For any non-decreasing $\rho: [0,1] \to \mathbb{R}$, we have a softmax attention $M:X_T \to [0,1]$ of the following form:
    \begin{equation}
        M(X_T) = f_2(\frac{\sum_{t=1}^{T}[\exp(\rho(x(t)))f_1(x(t))]}{\sum_{t=1}^{T}\exp(\rho(x(t)))})
    \end{equation}
    If $\max_{X_T \in \mathcal{X}^T}|M(X_T) - F_k(X_T)| \le \varepsilon$ and $k \ge 2$, then $f_1$ and $f_2$ have parameter count at least $\Omega(\varepsilon^{-o})$, where $o = \frac{T - k}{n+1} -1$. 
\end{theorem}

\paragraph{Proof Sketch:} To prove Theorem~\ref{theorem2}, we follow a similar idea to the first part of the proof of Theorem~\ref{theorem1}. 
\begin{proof}
   
Let
\[
    m:=T-k+1.
\]
The proof reduces the \(k\)-th largest-value task to a maximum-retrieval task over \(m\) free coordinates. The remaining \(k-1\) coordinates will be fixed at value \(1\). Since \(k\ge2\), at least one such fixed top coordinate is available.

Define
\[
    w(x):=\exp(\rho(x)).
\]
Since \(\rho\) is non-decreasing, for every \(x\in[0,1]\),
\[
    0<\frac{w(x)}{w(1)}\le 1.
\]
Set
\[
    \lambda(x):=\frac{w(x)}{w(1)},\qquad
    a(x):=\lambda(x) f_1(x)\in[0,1]^n,
    \qquad
    b(x):=\lambda(x)\in(0,1].
\]
Multiplying all attention weights by the same positive constant \(w(1)^{-1}\) does not change the softmax average. Hence the attention representation can be written as
\[
    A(X_T)
    :=
    \frac{\sum_{t=1}^{T} a(x(t))}
         {\sum_{t=1}^{T} b(x(t))}
    \in[0,1]^n.
\]

We now construct many candidate subsequences. Choose \(\varepsilon>0\) sufficiently small, for instance
\[
    0<\varepsilon<\frac{1}{64m}.
\]
Let
\[
    \Delta:=4\varepsilon,
    \qquad
    N:=\left\lfloor \frac{1}{16m\varepsilon}\right\rfloor .
\]
For each \(j=1,\dots,m\), define
\[
    \alpha_j:=\frac{j-1}{2m},
\]
and the grid
\[
    G_j
    :=
    \{\alpha_j+q\Delta:\ q=1,\dots,N\}.
\]
For sufficiently small \(\varepsilon\), each \(G_j\subset[0,1/2]\). Moreover, the grids are ordered: if \(i<j\), then every point of \(G_i\) is smaller than every point of \(G_j\). Also, two distinct points in the same \(G_j\) differ by at least
\[
    \Delta=4\varepsilon.
\]

Consider all grid subsequences
\[
    z=(z_1,\dots,z_m)\in G_1\times\cdots\times G_m.
\]
There are \(N^m\) such subsequences. For each \(z\), define
\[
    S(z)
    :=
    \left(
        \sum_{j=1}^{m} a(z_j),
        \sum_{j=1}^{m} b(z_j)
    \right)
    \in\mathbb{R}^{n+1}.
\]
Because \(a(z_j)\in[0,1]^n\) and \(b(z_j)\in(0,1]\), we have
\[
    S(z)\in[0,m]^{n+1}.
\]

Let
\[
    \eta:=4m\,N^{-m/(n+1)}.
\]
Partition the cube \([0,m]^{n+1}\) into axis-aligned cubes of side length \(\eta\). The number of cubes is at most
\[
    \left(\frac{m}{\eta}+1\right)^{n+1}
    =
    \left(\frac{1}{4}N^{m/(n+1)}+1\right)^{n+1}.
\]
For \(N\) large enough, this number is strictly smaller than \(N^m\). Hence, by the pigeonhole principle, there exist two distinct subsequences
\[
    z=(z_1,\dots,z_m),
    \qquad
    z'=(z'_1,\dots,z'_m)
\]
such that
\[
    \left\|
        \sum_{j=1}^{m} a(z_j)
        -
        \sum_{j=1}^{m} a(z'_j)
    \right\|_\infty
    \le \eta,
\]
and
\[
    \left|
        \sum_{j=1}^{m} b(z_j)
        -
        \sum_{j=1}^{m} b(z'_j)
    \right|
    \le \eta.
\]

Let
\[
    J:=\{j\in[m]:z_j\ne z'_j\}.
\]
Since \(z\ne z'\), the set \(J\) is nonempty. We now build two full input sequences \(X,Y\in[0,1]^T\). For the first \(k-1\) positions, set
\[
    x(1)=\cdots=x(k-1)=1,
    \qquad
    y(1)=\cdots=y(k-1)=1.
\]
For the remaining \(m=T-k+1\) positions, indexed by \(j=1,\dots,m\), set
\[
    x(k-1+j)
    :=
    \begin{cases}
        z_j, & j\in J,\\
        0, & j\notin J,
    \end{cases}
\]
and
\[
    y(k-1+j)
    :=
    \begin{cases}
        z'_j, & j\in J,\\
        0, & j\notin J.
    \end{cases}
\]

We first show that the targets are separated. Since the first \(k-1\) entries of both \(X\) and \(Y\) equal \(1\), and all remaining entries lie in \([0,1/2]\), the \(k\)-th largest entry is the maximum among the remaining \(m\) entries. Let
\[
    j_*:=\max J.
\]
Because the grids \(G_1,\dots,G_m\) are strictly ordered, the largest nonzero remaining entry of \(X\) is \(z_{j_*}\), and the largest nonzero remaining entry of \(Y\) is \(z'_{j_*}\). Therefore
\[
    F_k(X)=z_{j_*},
    \qquad
    F_k(Y)=z'_{j_*}.
\]
Since \(z_{j_*},z'_{j_*}\in G_{j_*}\) are distinct grid points,
\[
    |F_k(X)-F_k(Y)|
    =
    |z_{j_*}-z'_{j_*}|
    \ge \Delta
    =
    4\varepsilon.
\]

Next we show that the attention representations of \(X\) and \(Y\) are very close. Define
\[
    A_X:=\sum_{t=1}^{T}a(x(t)),
    \qquad
    B_X:=\sum_{t=1}^{T}b(x(t)),
\]
and similarly \(A_Y,B_Y\). Since \(b(1)=1\) and \(k\ge2\),
\[
    B_X\ge k-1,
    \qquad
    B_Y\ge k-1.
\]
Moreover, the coordinates where \(j\notin J\) are identical in \(X\) and \(Y\), and hence cancel in the difference. Thus
\[
    A_X-A_Y
    =
    \sum_{j\in J}\bigl(a(z_j)-a(z'_j)\bigr)
    =
    \sum_{j=1}^{m}\bigl(a(z_j)-a(z'_j)\bigr),
\]
so
\[
    \|A_X-A_Y\|_\infty\le \eta.
\]
Similarly,
\[
    |B_X-B_Y|\le \eta.
\]

The attention representations are
\[
    A(X)=\frac{A_X}{B_X},
    \qquad
    A(Y)=\frac{A_Y}{B_Y}.
\]
For each coordinate \(\ell=1,\dots,n\),
\[
\begin{aligned}
    \left|
        \frac{(A_X)_\ell}{B_X}
        -
        \frac{(A_Y)_\ell}{B_Y}
    \right|
    &\le
    \frac{|(A_X-A_Y)_\ell|}{B_X}
    +
    (A_Y)_\ell
    \left|
        \frac{1}{B_X}-\frac{1}{B_Y}
    \right|.
\end{aligned}
\]
Since \(0\le f_1(x)_\ell\le1\), we have \((A_Y)_\ell\le B_Y\). Therefore
\[
\begin{aligned}
    \left|
        \frac{(A_X)_\ell}{B_X}
        -
        \frac{(A_Y)_\ell}{B_Y}
    \right|
    &\le
    \frac{\eta}{k-1}
    +
    B_Y\frac{|B_X-B_Y|}{B_XB_Y}  \\
    &\le
    \frac{\eta}{k-1}
    +
    \frac{\eta}{k-1}
    =
    \frac{2\eta}{k-1}.
\end{aligned}
\]
Hence
\[
    \|A(X)-A(Y)\|_\infty
    \le
    \frac{2\eta}{k-1},
\]
and therefore
\[
    \|A(X)-A(Y)\|_2
    \le
    \frac{2\sqrt n\,\eta}{k-1}.
\]

Since \(M\) is assumed to \(\varepsilon\)-approximate \(F_k\), we have
\[
    |M(X)-F_k(X)|\le\varepsilon,
    \qquad
    |M(Y)-F_k(Y)|\le\varepsilon.
\]
Thus
\[
\begin{aligned}
    |M(X)-M(Y)|
    &\ge
    |F_k(X)-F_k(Y)|
    -
    |M(X)-F_k(X)|
    -
    |M(Y)-F_k(Y)| \\
    &\ge
    4\varepsilon-2\varepsilon
    =
    2\varepsilon.
\end{aligned}
\]
But
\[
    M(X)=f_2(A(X)),
    \qquad
    M(Y)=f_2(A(Y)).
\]
Therefore
\[
    |f_2(A(X))-f_2(A(Y))|
    \ge
    2\varepsilon.
\]

It remains to convert this separation into a parameter lower bound for \(f_2\). Write the two-layer tanh network \(f_2\) as
\[
    f_2(u)
    =
    \sum_{r=1}^{p}
        \alpha_r
        \tanh(\theta_r^\top u+\beta_r)
    +\gamma,
\]
where \(p\) is the hidden width. Since all weights are bounded by \(1\),
\[
    |\alpha_r|\le1,
    \qquad
    \|\theta_r\|_2\le\sqrt n.
\]
Because \(\tanh\) is \(1\)-Lipschitz,
\[
\begin{aligned}
    |f_2(u)-f_2(v)|
    &\le
    \sum_{r=1}^{p}
        |\alpha_r|
        \left|
            \tanh(\theta_r^\top u+\beta_r)
            -
            \tanh(\theta_r^\top v+\beta_r)
        \right| \\
    &\le
    \sum_{r=1}^{p}
        |\theta_r^\top(u-v)| \\
    &\le
    \sum_{r=1}^{p}
        \|\theta_r\|_2\|u-v\|_2 \\
    &\le
    p\sqrt n\,\|u-v\|_2.
\end{aligned}
\]
Applying this with \(u=A(X)\) and \(v=A(Y)\), we obtain
\[
    2\varepsilon
    \le
    p\sqrt n\,\|A(X)-A(Y)\|_2
    \le
    p\sqrt n\,
        \frac{2\sqrt n\,\eta}{k-1}.
\]
Hence
\[
    p
    \ge
    \frac{(k-1)\varepsilon}{n\eta}.
\]
Since the parameter count \(P_2\) of \(f_2\) is at least its hidden width \(p\),
\[
    P_2
    \ge
    \frac{(k-1)\varepsilon}{n\eta}.
\]

Finally, recall that
\[
    \eta=4m\,N^{-m/(n+1)}.
\]
Since
\[
    N=\left\lfloor \frac{1}{16m\varepsilon}\right\rfloor
    =
    \Theta(\varepsilon^{-1}),
\]
we have
\[
    \eta
    =
    O\!\left(
        \varepsilon^{m/(n+1)}
    \right).
\]
Therefore
\[
    P_2
    \ge
    C_{T,k,n}\,
    \varepsilon^{1-m/(n+1)}
    =
    C_{T,k,n}\,
    \varepsilon^{-\left(\frac{m}{n+1}-1\right)}.
\]
Since \(m=T-k+1\), this gives
\[
    P_2
    =
    \Omega\!\left(
        \varepsilon^{-\left(\frac{T-k+1}{n+1}-1\right)}
    \right).
\]
The total parameter count of \((f_1,f_2)\) is at least \(P_2\), so the same lower bound holds for the combined parameter count.

This proves the theorem.

\end{proof}
\begin{remark}
    We choose $\rho$ to be non-decreasing in the spirit that $x(\sigma(k))$ doesn't happen to be the maximizer of another function. 
\end{remark}

\subsection{Justification of the Parameter Cost in Section ~\ref{sec: parameter_cost}}\label{app:thm3}
We propose the formal version of Theorem~\ref{theorem3} to justify the choice of parameter cost in Section ~\ref{sec: parameter_cost}. 
\begin{theorem}[Formal Version of Theorem~\ref{theorem3}]
    For sequence $X_T$, let $T_0$ be a integer such that $1\le T_0\le T$. Then for two layer Tanh-activated neural networks $f_1 : \mathcal{X}^{T_0} \to [0,1]^n$ and $f_2: [0,1]^n \to \mathcal{X}^{T_0}$, let $I \subseteq [T]$ be an index set with $|I| = T_0$. We have the following: 
    \begin{itemize}
        \item For a fixed $p \ge 1$, there exists $f_1, f_2$ with total parameter count of order $O(\varepsilon^{-\frac{T_0d}{n}})$ such that $\|X(I) - f_2 \circ f_1 (X(I))\|_{L_p} \le \varepsilon$.
        \item If $\|X(I) - f_2 \circ f_1 (X(I))\|_{L_\infty} \le \varepsilon$ and the norm of the weights of $f_1$ and $f_2$ are bounded by $B$, then $f_1, f_2$ must have total parameter count of order $O(\varepsilon^{\min(1-\frac{T_0d}{n}, 0)})$.
    \end{itemize}
\end{theorem}
\paragraph{Explanation:} In practical Transformers, information associated with retrieved positions is not represented explicitly, but is instead embedded in the token states $x_l(t)\in\mathbb{R}^{E_l}$. Recovering this information, namely $X(I(t,l))$, generally requires nontrivial decoding by the feed-forward layers. Theorem~\ref{theorem3} have demonstrated that to achieve $\varepsilon$-accurate encoding and decoding, a parameter cost exponential to $|I(t,l)|/E_l$ is required. It is also shown that $O(1/\varepsilon^{|I(t,l)|d/E_l})$ parameters are enough for such encoding and decoding. Therefore $O(1/\varepsilon^{|I(t,l)|d/E_l})$ is a reasonable parameter cost associated to the processing of information set $I(t,l)$. 

We then consider proving Theorem~\ref{theorem3}. 

\paragraph{Proof Sketch} For the first part of Theorem~\ref{theorem3}, we first construct two Heaviside function activated FFNs $\hat f_1, \hat f_2$ to store the information of $X(I)$ using binary representation up to $\varepsilon/1$-accuracy. Then we replace each Heaviside unit with Tanh unit to derive the $\varepsilon$ approximation in $L_p$ norm. For the second part, we use Pigeonhole Principle to complete the argument. 
\begin{proof}
\textbf{First Part of Theorem~\ref{theorem3}} We first consider the case $T_0d > n$. When we say Heaviside function in this proof, we refer to $H(x) = \bf{1}_{\{x\ge 0\}}$. \\
\textbf{Construction of Encoding function $\hat f_1$.} 
Let \(m=T_0d\), and identify \(X(I)\in \mathcal{X}^{T_0}\) with a vector
\[
    V=(V_1,\dots,V_m)\in [0,1]^m .
\]
Choose
\[
    L=\left\lceil \log_2\frac{2m}{\varepsilon}\right\rceil .
\]
For each coordinate \(V_j\), let
\[
    V_j^{(L)}=\sum_{\ell=1}^{L} 2^{-\ell} b_{j,\ell}(V_j),
    \qquad b_{j,\ell}(V_j)\in\{0,1\},
\]
be its \(L\)-bit binary truncation. Then
\[
    \|V-V^{(L)}\|_{\infty}\le 2^{-L}\le \frac{\varepsilon}{2m},
    \qquad
    \|V-V^{(L)}\|_{p}\le m^{1/p}\frac{\varepsilon}{2m}\le \frac{\varepsilon}{2}.
\]

For each fixed coordinate \(V_j\), the map
\[
    V_j \longmapsto (b_{j,1},\ldots,b_{j,L})
\]
can be implemented by a two-layer Heaviside network with \(O(2^L)\) parameters. Indeed, the \(L\)-bit truncation partitions \([0,1]\) into \(2^L\) dyadic intervals, and the binary vector \((b_{j,1},\ldots,b_{j,L})\) is constant on each such interval. A two-layer Heaviside network can represent this piecewise-constant map by using threshold units to identify the dyadic intervals and then assigning to each interval its corresponding binary code. Therefore, encoding all \(m=T_0d\) coordinates requires
\[
    O(m2^L)
\]
parameters.

Since
\[
    L=\left\lceil \log_2\frac{2m}{\varepsilon}\right\rceil ,
\]
we have \(2^L=O(m/\varepsilon)\). Hence, for fixed \(T_0\) and \(d\), the encoder can be implemented by a two-layer Heaviside network \(\widehat f_1:[0,1]^m\to [0,1]^n\) with
\[
    O(m2^L)=O\!\left(\frac{1}{\varepsilon}\right)
\]
parameters. The resulting \(mL\) binary digits are then distributed among the \(n\) latent coordinates, with each coordinate storing at most
\[
    \left\lceil \frac{mL}{n}\right\rceil
\]
digits through its own binary expansion.

\textbf{Construction of Decoding Function $\hat f_2$.} Next we construct the decoder \(\widehat f_2\). \\
Recall that the encoder \(\widehat f_1\) allocates the \(mL\) binary digits among the \(n\) latent coordinates, with each latent coordinate storing at most
\[
    q:=\left\lceil \frac{mL}{n}\right\rceil
\]
digits through a binary expansion. The decoder \(\widehat f_2\) is constructed according to this fixed allocation rule. For each latent coordinate, recovering its assigned binary digits amounts to decoding a \(q\)-bit binary expansion. This is again a piecewise-constant map on \([0,1]\) with at most \(2^q\) dyadic intervals, and hence can be implemented by a two-layer Heaviside network with \(O(2^q)\) parameters. Applying this to all \(n\) latent coordinates and then rearranging the recovered bits according to the allocation rule of \(\widehat f_1\), the decoder recovers all digits
\[
    \{b_{j,\ell}:1\le j\le m,\ 1\le \ell\le L\}.
\]
It then outputs
\[
    V_j^{(L)}=\sum_{\ell=1}^{L}2^{-\ell}b_{j,\ell},
    \qquad j=1,\dots,m .
\]
Therefore,
\[
    \widehat f_2\circ \widehat f_1(V)=V^{(L)} .
\]
By the choice of
\[
    L=\left\lceil \log_2\frac{2m}{\varepsilon}\right\rceil ,
\]
we have
\[
    \|V-V^{(L)}\|_{\infty}\le 2^{-L}\le \frac{\varepsilon}{2}.
\]

Moreover,
\[
    2^q
    =
    2^{\left\lceil \frac{mL}{n}\right\rceil}
    \le
    2\left(2^L\right)^{m/n}.
\]
Since \(2^L=O(1/\varepsilon)\), the decoder \(\widehat f_2\) has parameter count
\[
    O\!\left(2^{\left\lceil \frac{mL}{n}\right\rceil}\right)
    =
    O\!\left(\varepsilon^{-m/n}\right).
\]

We then consider replacing these Heaviside activation functions with Tanh activation functions. We first prove the following lemma regarding such replacements.
\begin{lemma}[Replacing Heaviside Unit with Tanh Unit] 
Let \(p>1\), let \(a<0<b\), and let \(\mu\) be a finite positive measure on \([a,b]\).
Let \(H(x)=\mathbf{1}_{\{x\ge 0\}}\). Then for every \(\delta>0\), there exist
\(A,B,C,D\in\mathbb{R}\) such that
\[
    \left\|H-\bigl(A\tanh(Bx+C)+D\bigr)\right\|_{L_p(\mu)}<\delta .
\]
\end{lemma}

\begin{proof}
It suffices to consider functions of the form
\[
    \phi_{B,C}(x):=\frac{1+\tanh(Bx+C)}{2}.
\]
Let \(M:=\mu([a,b])<\infty\). If \(M=0\), the claim is trivial, so assume \(M>0\).

Since \(\mu\) is finite,
\[
    \mu\bigl((-\eta,\eta)\setminus\{0\}\bigr)\to 0
    \qquad \text{as } \eta\downarrow 0 .
\]
Choose \(\eta>0\) such that
\[
    \mu\bigl((-\eta,\eta)\setminus\{0\}\bigr)<\frac{\delta^p}{3}.
\]
Next choose \(C>0\) such that
\[
    \left|1-\frac{1+\tanh C}{2}\right|^p \mu(\{0\})<\frac{\delta^p}{3}.
\]
Finally, choose \(R>0\) sufficiently large so that
\[
    \left(\frac{1+\tanh(-R)}{2}\right)^p M
    +
    \left(1-\frac{1+\tanh R}{2}\right)^p M
    <\frac{\delta^p}{3},
\]
and then choose \(B>0\) large enough that
\[
    -B\eta+C\le -R,
    \qquad
    B\eta+C\ge R .
\]
For \(x\le -\eta\), we have \(Bx+C\le -R\), and hence \(\phi_{B,C}(x)\) is uniformly small. For
\(x\ge \eta\), we have \(Bx+C\ge R\), and hence \(\phi_{B,C}(x)\) is uniformly close to \(1\).
On \((-\eta,\eta)\setminus\{0\}\), the pointwise error is bounded by \(1\), while at \(x=0\) it is controlled by the choice of \(C\). Therefore,
\[
\begin{aligned}
    \|H-\phi_{B,C}\|_{L_p(\mu)}^p
    &\le
    \int_{[a,-\eta]} |\phi_{B,C}(x)|^p\,d\mu(x)
    +\int_{[\eta,b]} |1-\phi_{B,C}(x)|^p\,d\mu(x)  \\
    &\quad
    +\mu\bigl((-\eta,\eta)\setminus\{0\}\bigr)
    +|1-\phi_{B,C}(0)|^p\mu(\{0\})  \\
    &<\delta^p .
\end{aligned}
\]
Thus \(\|H-\phi_{B,C}\|_{L_p(\mu)}<\delta\). Taking
\[
    A=\frac12,\qquad D=\frac12
\]
gives the desired form.
\end{proof}

We now pass from Heaviside networks to Tanh networks. Let
\[
    \widehat F:=\widehat f_2\circ \widehat f_1
\]
be the Heaviside network constructed above, and let \(N_0\) be the total number of Heaviside units in \(\widehat F\). We replace these Heaviside units one at a time.

At any intermediate stage, if not all Heaviside units have been replaced, choose a Heaviside unit whose output does not feed into any other remaining Heaviside unit. Such a unit always exists by the acyclicity of the feed-forward computation graph. Writing the output of this unit as
\[
    H(z),
\]
where \(z\) is an affine function of the previous layer outputs, the part of the network after this unit is a Tanh/affine network, and is therefore continuous, indeed Lipschitz, with respect to this scalar input on the relevant compact range. Denote this downstream map by \(G\), and let \(L_G\) be a Lipschitz constant.

Let \(\nu\) be the pushforward of the input distribution under \(z\). By the preceding lemma, for any \(\delta>0\), there exist \(A,B,C,D\in\mathbb{R}\) such that
\[
    \|H-(A\tanh(B\cdot+C)+D)\|_{L_p(\nu)}<\delta .
\]
Choose \(\delta\) sufficiently small so that
\[
    L_G\delta \le \frac{\varepsilon}{2N_0}.
\]
Replacing this Heaviside unit by
\[
    A\tanh(Bz+C)+D
\]
therefore changes the final network output by at most
\[
    \frac{\varepsilon}{2N_0}
\]
in \(L_p\).

Repeating this procedure for all \(N_0\) Heaviside units yields Tanh-activated networks \(f_1,f_2\) satisfying
\[
    \|f_2\circ f_1-\widehat f_2\circ \widehat f_1\|_{L_p}
    \le
    \sum_{r=1}^{N_0}\frac{\varepsilon}{2N_0}
    =
    \frac{\varepsilon}{2}.
\]
Since the Heaviside construction gives
\[
    \|V-\widehat f_2\circ \widehat f_1(V)\|_{L_p}
    \le \frac{\varepsilon}{2},
\]
the triangle inequality implies
\[
    \|V-f_2\circ f_1(V)\|_{L_p}
    \le
    \|V-\widehat f_2\circ \widehat f_1(V)\|_{L_p}
    +
    \|\widehat f_2\circ \widehat f_1(V)-f_2\circ f_1(V)\|_{L_p}
    \le \varepsilon .
\]
Thus the desired \(L_p\)-approximation is obtained by Tanh-activated networks. 

When $n \ge T_0d$, we just let $f_1(V) = (V, 0) \in \mathbb{R}^n$, and $f_2((U,0)) = U$, then we can achieve exact representation with $O(1)$ parameters. 

\textbf{Second Part of Theorem~\ref{theorem3}}

Let \(m=T_0d\). The case \(m\le n\) only gives the trivial lower bound
\(\Omega(1)\), so assume \(m>n\). By the standard packing argument for
\(f_1:[0,1]^m\to[0,1]^n\), there exist \(V_1,V_2\in[0,1]^m\) such that
\[
    \|V_1-V_2\|_\infty\ge 4\varepsilon,
    \qquad
    \|f_1(V_1)-f_1(V_2)\|_2\le C\varepsilon^{m/n}.
\]
Since \(\|V-f_2\circ f_1(V)\|_{L_\infty}\le \varepsilon\), we have
\[
\begin{aligned}
    \|f_2(f_1(V_1))-f_2(f_1(V_2))\|_2
    &\ge
    \|V_1-V_2\|_\infty
    -\|V_1-f_2(f_1(V_1))\|_\infty  \\
    &\quad
    -\|V_2-f_2(f_1(V_2))\|_\infty
    \ge 2\varepsilon .
\end{aligned}
\]
Applying the Lemma 4 from \cite{yu2026.EffectAttentionHead} to \(f_2\), with
\[
    A=C\varepsilon^{m/n},
    \qquad
    B=2\varepsilon,
\]
shows that \(f_2\) must have at least
\[
    \Omega\!\left(\frac{B}{A\sqrt n}\right)
    =
    \Omega\!\left(\varepsilon^{1-m/n}\right)
    =
    \Omega\!\left(\frac{1}{\varepsilon^{m/n-1}}\right)
\]
parameters. Therefore the total parameter count of \(f_1,f_2\) is at least
\[
    \Omega\!\left(\varepsilon^{1-\frac{T_0d}{n}}\right),
\]
which is the desired lower bound in the nontrivial regime \(T_0d>n\).
\end{proof}
\begin{remark}
    We can replace Tanh with any $1$-Lipschitz sigmoidal activation function to obtain the same results following similar proof strategy. 
\end{remark}

\section{Derivation of Predictions}
\subsection{Derivation of Intrinsic Dimension Phenomena in Two-Layer Transformers}\label{app: lemma1}
\paragraph{Derivation of Proposition~\ref{lemma:NOCofIntrinsic}} To prove the upper bound, we construct the following $D$ Trees of Comparison $TR_1, \dots, TR_D$. 
\begin{enumerate}
    \item For $TR_i$, each leaf is associated with an ordered index set ${t_1, t_2}$ for $1 \le t_1, t_2 \le T$, thus $TR_i$ has $T^2$ leaves. And for each internal node of $TR_i$, its corresponding $f_{in}$ is defined as $f_{in} (X[\{s_1, s_2\}]) = x(s_1)^T A_i x(s_2)$. Thus $I_{fi}(X_T, TR_i) = \argmax_{\{s_1, s_2\}} x(s_1)^T A_i x(s_2)$. Then we have $I_F(X_T) \subseteq \bigcup_{i=1}^DI_{fi}(X_T, TR_i)$.  
\end{enumerate}

With this upper bound we also have $\beta' \le \beta_1 \le 2$. Then following a similar argument in \ref{app:Triangle_explain}, we can show that each index set $\{t_1, t_2\}$ for $t_1 \ne t_2$ must be on the leaf of a Tree of Comparison. This gives a lower bound $\Omega(T^2)$ for the Number of Comparison, and hence $\beta_1 = \beta' = 2$. 

As the target $F$ have $\beta_1 = \beta' = 2$. Using only max position aggregation, we have that $I(t,1) \subseteq \bigcup_{i=1}^{h_1} I(s_i(t),0) \cup I(t,0)$. As $|I(s,0)| = |I(t,0)| = 1$, we have that $|I(t,1)| \le (h_1 + 1)$. Similarily we have $|I(T+1,1)| \le h_1$. Then $|I(T+1, 2)| \le \sum_{i=1}^{h_2} |I(s_i(T+1), 1)| + |I(T+1,1)| \le h_1h_2 + h_2 + h_1 < (h_1+1)(h_2+1)$. Substituting all terms into the formula for Number of Comparison, in the regime of $h_1, h_2 \ll T$, we have
\begin{align}
    &\sum_{t=1}^{T} \sum_{l=1}^{L-1} \Bigl(|I(t,l)|^{\beta'} - 1 + \sum_{i=1}^{h_l}(T-1)\Bigr) + \sum_{l=1}^{L} \Bigl(|I(T+1,l)|^{\beta'} - 1 + \sum_{i=1}^{h_l}(T-1)\Bigr)\\
    \le& \sum_{t=1}^{T} \Bigl(|I(t,1)|^2 - 1 + \sum_{i=1}^{h_1}(T-1)\Bigr) + \sum_{l=1}^{2} \Bigl(|I(T+1,l)|^2 - 1 + \sum_{i=1}^{h_l}(T-1)\Bigr)\\
    \le& T(h_1+1)^2 -T + h_1T(T-1) + h_1^2 + (h_1+1)^2(h_2+1)^2 - 2 +2(T-1) \\
    =& h_1T^2 + O(T)
\end{align}
Therefore we need $h_1 \ge D$ to have large enough Number of Comparison. 

To show that the $(D,D)$ heads model exists, we construct an explicit \infoflow model with $h_1 = h_2 = D$ heads that learns the target.
We set 
\begin{equation}
    I(t,1) = \{s_i : s_i = \arg\max_{1 \le s \le T} x(s)^\top A_i x(t)\} \cup \{t\}
\end{equation}
and 
\begin{equation}
   I(T+1,2) = \{(t,s_i(t)) : t \in \bigcup_{i=1}^D \{ \arg\max_{1 \le t \le T} x(s_i(t))^\top A_i x(t) \}\}.
\end{equation}
We have $I_F(X_T) \subseteq I(T+1,2)$,
confirming that the two-layer Transformer with $(D,D)$ heads approximates $F$ effectively.

\subsection{Derivation of Approximation Hardness of Higher-Order Retrieval Tasks}\label{app:lemma2}

\paragraph{Derivation of Proposition~\ref{lemma:triangle}} Suppose the target $F$ have Number of Comparison lower bounded by $C_1T^{\beta'}$ for a positive constant $C_1$ (This is valid by the definition of Order of Comparison $\beta'$). Let $M = \max_{t,l} |I(t,l)|$, we have an estimation for the Number of Comparison of the model (in the regime of $H,L,E\ll T$):
\begin{align}
    &\sum_{t=1}^{T} \sum_{l=1}^{L-1} \Bigl(|I(t,l)|^{\beta_1} - 1 + \sum_{i=1}^{h_l}(T-1)\Bigr) + \sum_{l=1}^{L} \Bigl(|I(T+1,l)|^{\beta_1} - 1 + \sum_{i=1}^{h_l}(T-1)\Bigr)\\
    \le& \sum_{t=1}^{T} \sum_{l=1}^{L-1} \Bigl(M^{\beta_1} - 1 + H(T-1)\Bigr) + \sum_{l=1}^{L} \Bigl(M^{\beta_1} - 1 + H(T-1)\Bigr) \\
    \le& T(L-1)M^{\beta_1} + T(L-1)H(T-1) + LM^{\beta_1} + LH(T-1) \\
    \le& TLM^{\beta_1} + HLT^2  +LM^{\beta_1} + LHT \\
    =& TLM^{\beta_1} + HLT^2 + O(M^{\beta_1}+T)
\end{align}
If $TLM^{\beta_1} + HLT^{\beta_1} + O(M^{\beta_1}+T) \ge C_1T^{\beta'}$ holds for $T \gg HL$, we need $TLM^{\beta_1} \ge C_2T^{\beta'}$ for positive $C_2$, thus we have $M \ge CT^{\frac{\beta'-1}{\beta_1}}/L$ for positive constant $C$. This leads to a parameter cost of $O(1/\varepsilon^{M/E}) = O(1/\varepsilon^{CT^{\frac{\beta'-1}{\beta_1}}/LE})$.

\iffalse
\begin{Proposition}[Informal]\label{lemma:triangle}
For the target $F$ defined in \eqref{eq:triangle-target},
let $M = \max_{t,l} |I(t,l)|$ for an $L$-layer $H$-heads-per-layer Transformer with embedding dimension $E$.
For the transformer's number of comparison to match $T^3 - 1$,
it is necessary that $M \ge \frac{T^{2/3}}{L}$,
leading to a parameter cost of $O(\varepsilon^{-\frac{T^{2/3}}{LE}})$ for $\varepsilon$-approximation of $F$.
\end{Proposition}
\fi

\section{Experimental Details and Results}
\label{app:exp1}

\subsection{Max-Position Retrieval in Trained Two-Layer Transformers}\label{app:exp_max_retrieval}

\paragraph{Setup.}
We train two-layer Transformers on the generalized $D$-retrieval target with intrinsic dimension $D = 6$,
sequence length $T = 64$,
and feed-forward hidden width $w_1 = 384$.
We consider six configurations of head counts $(h_1, h_2)$ and corresponding embedding dimensions $E_1 = h_1 n_1$,
with per-head embedding dimension $n_1$ fixed such that $E_1$ is proportional to $h_1$.
For each configuration,
we train the model to convergence and record the best random seed.
To assess whether the trained model retains the information of the argmax token at each head,
we fit a linear map $g_i : \mathbb{R}^{E_1} \to \mathbb{R}^d$ by regression for each head $i = 1, \dots, h_1$,
and measure the normalized recovery error
\begin{equation}
\mathcal{E}_i = \mathbb{E}_{X_T} \mathbb{E}_t
\frac{|g_i(x_1'(t)) - x(s_i(t))|^2}{|x(s_i(t))|^2},
\label{eq:recovery-error}
\end{equation}
where $s_i(t) = \arg\max_s A_i(x(t), x(s))$ is the argmax position of head $i$ at position $t$.
We report $\mathcal{E}_i$ for both the initialized and trained model,
as well as the ratio of trained to initialized error,
which measures the degree to which training improves argmax position recovery.

\paragraph{Results.}
The recovery errors for each configuration are reported in Tables~\ref{tab:recovery1} to \ref{tab:recovery2nd3}.
In all configurations,
training reduces the recovery error substantially compared to initialization,
with trained-to-initialized ratios ranging from $0.08$ to $0.84$.
The reduction is most pronounced in configurations where the number of first-layer heads $h_1$ matches or exceeds the intrinsic dimension $D = 6$,
consistent with the InfoFlow prediction that $h_1 \ge D$ heads are needed for effective approximation of the target.

\begin{table}[h]
\centering
\caption{Recovery errors for $h_1 = 6$, $h_2 = 6$,
embedding dimension $E_1 = 36$, seed $12$.}
\label{tab:recovery1}
\begin{tabular}{|c|c|c|c|c|c|c|}
\hline
& Head 1 & Head 2 & Head 3 & Head 4 & Head 5 & Head 6 \\
\hline
Initialized & 0.3497 & 0.3546 & 0.3233 & 0.3738 & 0.4851 & 0.3496 \\
\hline
Trained & 0.0302 & 0.0323 & 0.0310 & 0.0305 & 0.0315 & 0.0317 \\
\hline
Trained/Initialized & 0.0863 & 0.0910 & 0.0958 & 0.0817 & 0.0650 & 0.0906 \\
\hline
\end{tabular}
\end{table}

\begin{table}[h]
\centering
\caption{Recovery errors for $h_1 = 5$, $h_2 = 5$,
embedding dimension $E_1 = 40$, seed $9$.}
\label{tab:recovery2}
\begin{tabular}{|c|c|c|c|c|c|}
\hline
& Head 1 & Head 2 & Head 3 & Head 4 & Head 5 \\
\hline
Initialized & 0.3496 & 0.3711 & 0.3676 & 0.3113 & 0.3836 \\
\hline
Trained & 0.1816 & 0.1735 & 0.2433 & 0.2373 & 0.2350 \\
\hline
Trained/Initialized & 0.5195 & 0.4676 & 0.6618 & 0.7624 & 0.6126 \\
\hline
\end{tabular}
\end{table}

\begin{table}[h]
\centering
\caption{Recovery errors for $h_1 = 5$, $h_2 = 6$,
embedding dimension $E_1 = 40$, seed $8$.}
\label{tab:recovery3}
\begin{tabular}{|c|c|c|c|c|c|}
\hline
& Head 1 & Head 2 & Head 3 & Head 4 & Head 5 \\
\hline
Initialized & 0.3167 & 0.3248 & 0.3413 & 0.3234 & 0.3261 \\
\hline
Trained & 0.2099 & 0.2531 & 0.1160 & 0.1382 & 0.1405 \\
\hline
Trained/Initialized & 0.6629 & 0.7793 & 0.3400 & 0.4273 & 0.4308 \\
\hline
\end{tabular}
\end{table}

\paragraph{Discussion.}
The results across all configurations consistently show that training improves the recoverability of argmax position information from the post-attention token state.
The improvement is most pronounced when $h_1 \ge D$,
where the trained-to-initialized ratio drops to as low as $0.08$,
indicating near-perfect recovery after training.
When $h_1 < D$,
the improvement is more modest,
which is consistent with the InfoFlow prediction that insufficient head count prevents the model from independently tracking each of the $D$ relevant argmax positions.

\subsection{Second-Largest Position Retrieval in Trained Two-Layer Transformers}\label{app:exp_second_cannot}

\paragraph{Setup.}
Using the same experimental configurations as Section~\ref{app:exp1},
we assess whether trained two-layer Transformers retain the information 
of the token attaining the second-largest attention score at each head.

\paragraph{Results.}
The recovery errors for each configuration are reported in 
Tables~\ref{tab:recovery2nd1} to \ref{tab:recovery2nd7}.

\begin{table}[h]
\centering
\caption{Second-largest position recovery errors for $h_1 = 6$,
$h_2 = 6$,
embedding dimension $E_1 = 36$,
seed $12$.}
\label{tab:recovery2nd1}
\begin{tabular}{|c|c|c|c|c|c|c|}
\hline
& Head 1 & Head 2 & Head 3 & Head 4 & Head 5 & Head 6 \\
\hline
Initialized & 0.4085 & 0.4126 & 0.3843 & 0.4303 & 0.5318 & 0.4086 \\
\hline
Trained & 0.4061 & 0.4154 & 0.4176 & 0.4191 & 0.4132 & 0.4148 \\
\hline
Trained/Initialized & 0.9940 & 1.0067 & 1.0864 & 0.9739 & 0.7770 & 1.0151 \\
\hline
\end{tabular}
\end{table}

\begin{table}[h]
\centering
\caption{Second-largest position recovery errors for $h_1 = 5$,
$h_2 = 5$,
embedding dimension $E_1 = 40$,
seed $9$.}
\label{tab:recovery2nd2}
\begin{tabular}{|c|c|c|c|c|c|}
\hline
& Head 1 & Head 2 & Head 3 & Head 4 & Head 5 \\
\hline
Initialized & 0.4093 & 0.4272 & 0.4249 & 0.3746 & 0.4390 \\
\hline
Trained & 0.4095 & 0.4099 & 0.4097 & 0.4082 & 0.4083 \\
\hline
Trained/Initialized & 1.0005 & 0.9594 & 0.9641 & 1.0895 & 0.9299 \\
\hline
\end{tabular}
\end{table}

\begin{table}[h]
\centering
\caption{Second-largest position recovery errors for $h_1 = 5$,
$h_2 = 6$,
embedding dimension $E_1 = 40$,
seed $8$.}
\label{tab:recovery2nd3}
\begin{tabular}{|c|c|c|c|c|c|}
\hline
& Head 1 & Head 2 & Head 3 & Head 4 & Head 5 \\
\hline
Initialized & 0.3779 & 0.3852 & 0.4013 & 0.3849 & 0.3873 \\
\hline
Trained & 0.4018 & 0.4038 & 0.4085 & 0.4050 & 0.4068 \\
\hline
Trained/Initialized & 1.0631 & 1.0483 & 1.0180 & 1.0523 & 1.0505 \\
\hline
\end{tabular}
\end{table}

\begin{table}[h]
\centering
\caption{Second-largest position recovery errors for $h_1 = 5$,
$h_2 = 7$,
embedding dimension $E_1 = 40$,
seed $15$.}
\label{tab:recovery2nd4}
\begin{tabular}{|c|c|c|c|c|c|}
\hline
& Head 1 & Head 2 & Head 3 & Head 4 & Head 5 \\
\hline
Initialized & 0.3854 & 0.3954 & 0.3863 & 0.4124 & 0.3869 \\
\hline
Trained & 0.4044 & 0.4077 & 0.4017 & 0.4109 & 0.3996 \\
\hline
Trained/Initialized & 1.0493 & 1.0313 & 1.0399 & 0.9963 & 1.0329 \\
\hline
\end{tabular}
\end{table}

\begin{table}[h]
\centering
\caption{Second-largest position recovery errors for $h_1 = 6$,
$h_2 = 5$,
embedding dimension $E_1 = 36$,
seed $1$.}
\label{tab:recovery2nd5}
\begin{tabular}{|c|c|c|c|c|c|c|}
\hline
& Head 1 & Head 2 & Head 3 & Head 4 & Head 5 & Head 6 \\
\hline
Initialized & 0.4025 & 0.3800 & 0.4280 & 0.4466 & 0.4230 & 0.3976 \\
\hline
Trained & 0.4045 & 0.4027 & 0.4032 & 0.4042 & 0.4028 & 0.4085 \\
\hline
Trained/Initialized & 1.0049 & 1.0598 & 0.9420 & 0.9049 & 0.9524 & 1.0272 \\
\hline
\end{tabular}
\end{table}

\begin{table}[h]
\centering
\caption{Second-largest position recovery errors for $h_1 = 7$,
$h_2 = 5$,
embedding dimension $E_1 = 42$,
seed $13$.}
\label{tab:recovery2nd6}
\begin{tabular}{|c|c|c|c|c|c|c|c|}
\hline
& Head 1 & Head 2 & Head 3 & Head 4 & Head 5 & Head 6 & Head 7 \\
\hline
Initialized & 0.3813 & 0.3905 & 0.4137 & 0.3863 & 0.3812 & 0.4053 & 0.3964 \\
\hline
Trained & 0.4001 & 0.4026 & 0.4046 & 0.4042 & 0.3996 & 0.4067 & 0.4029 \\
\hline
Trained/Initialized & 1.0493 & 1.0310 & 0.9781 & 1.0461 & 1.0483 & 1.0034 & 1.0163 \\
\hline
\end{tabular}
\end{table}

\begin{table}[h]
\centering
\caption{Second-largest position recovery errors for $h_1 = 7$,
$h_2 = 7$,
embedding dimension $E_1 = 42$,
seed $10$.}
\label{tab:recovery2nd7}
\begin{tabular}{|c|c|c|c|c|c|c|c|}
\hline
& Head 1 & Head 2 & Head 3 & Head 4 & Head 5 & Head 6 & Head 7 \\
\hline
Initialized & 0.4740 & 0.3949 & 0.3801 & 0.3922 & 0.3955 & 0.4124 & 0.4001 \\
\hline
Trained & 0.4062 & 0.4148 & 0.4147 & 0.4054 & 0.4098 & 0.3903 & 0.3868 \\
\hline
Trained/Initialized & 0.8569 & 1.0506 & 1.0909 & 1.0337 & 1.0364 & 0.9464 & 0.9668 \\
\hline
\end{tabular}
\end{table}

\paragraph{Discussion.}
In sharp contrast to the argmax recovery experiment,
training does not improve the ability to recover the second-largest position information across any configuration.
The trained-to-initialized ratios are uniformly close to $1.0$ across all heads and configurations,
indicating that the trained model retains no more information about the second-largest attention position than an untrained model.
This provides empirical confirmation that softmax attention concentrates exclusively on the argmax position,
and that the max-position retrieval mechanism identified in Theorem~\ref{theorem2} and Section~\ref{sec:mechanisms} is not merely a theoretical phenomenon but is precisely what emerges in trained networks.

\subsection{Specific Position Aggregation Experiment}
\label{app:exp2}

\paragraph{Setup.}
We consider sequences $X_T = (x(1), \dots, x(T))$ with $T = 10$ and token dimension $d = 2$,
where each token $x(t) \overset{\mathrm{i.i.d.}}{\sim} \mathcal{N}(0, I_d)$.
The target function is
\begin{equation}
F(X_T) = x(1) + x(2) + x(3) \in \mathbb{R}^d,
\end{equation}
namely the sum of the first three tokens,
which is independent of the remaining $T - 3$ tokens and of the specific values at positions $1$, $2$, $3$.

\paragraph{Model.}
We train a one-layer one-head Transformer with learned positional encoding.
Each token $x(t)$ is first mapped to a $d_{\mathrm{model}}$-dimensional embedding via a linear layer,
and a learned positional embedding $p(t) \in \mathbb{R}^{d_{\mathrm{model}}}$ is added.
A learnable CLS token,
initialized to zero in embedding space and assigned positional embedding $p(T+1)$,
is appended to the sequence.
The model applies one round of single-head attention with a residual connection and layer normalization,
then reads out the CLS token through a two-layer MLP to produce the prediction $\hat{F}(X_T) \in \mathbb{R}^d$.

\paragraph{Training.}
We train on $N = 50{,}000$ i.i.d. samples with MSE loss using AdamW with learning rate $10^{-3}$ and weight decay $10^{-2}$,
with a cosine annealing schedule over $1000$ epochs followed by $1000$ epochs at learning rate $10^{-4}$.
Validation is performed on a held-out set of $10{,}000$ samples.

\paragraph{Results.}
The model achieves near-zero validation MSE.
Inspecting the learned attention weights,
the CLS token's mean attention over the validation set concentrates with weight $1/3$ on each of positions $1$, $2$, $3$ and weight $0$ on all remaining positions.
This pattern is stable across samples regardless of token values,
demonstrating that the model implements position-specific selection purely through the positional encoding rather than through content-dependent attention.
The attention weights are shown in Figure~\ref{fig:pos-agg}.

\subsection{Intrinsic Dimension Experiment}
\label{app:exp:intrinsic}

\paragraph{Target.}
We train two-layer Transformers on the permutation-invariant target
\begin{equation}
F(X_T) = \sum_{i=1}^D \max_{1 \le s,t \le T} x(s)^\top A_i x(t),
\end{equation}
where $A_i \in \mathbb{R}^{d \times d}$ are distinct randomly generated 
matrices fixed throughout training.
We consider intrinsic dimensions $D \in \{2, 3, 4, 5, 6\}$ with 
sequence length $T = 64$ and token dimension $d = 4$.

\paragraph{Model.}
For each value of $D$,
we train two-layer Transformers with three head configurations:
$(h_1, h_2) = (D, D)$,
$(h_1, h_2) = (D-1, D)$,
and $(h_1, h_2) = (D, D-1)$.
The total embedding dimension is fixed at $E_1 = \lceil\frac{6D}{h_1}\rceil h_1,  E_2 = \lceil\frac{6D}{h_2}\rceil h_2$.
The feed-forward hidden width is $w_1 = 16D$.
No positional encoding is used,
consistent with the permutation-invariant nature of the target.

\paragraph{Training.}
Each configuration is trained for $16$ independent runs with different 
random seeds.
We use the AdamW optimizer with learning rate $10^{-3}$ and weight 
decay $10^{-2}$,
with a cosine annealing schedule.
Training is performed on $N = 36{,}000$ i.i.d. samples drawn from 
$x(t) \overset{\mathrm{i.i.d.}}{\sim} \mathcal{N}(0, I_d)$,
with a held-out validation set of $4{,}000$ samples.
We report the best validation NMSE across training,
averaged over all $16$ seeds.

\paragraph{Results.}
The mean best validation NMSE for each configuration and value of $D$ 
is reported in Table~\ref{tab:intrinsic}.
Results are also shown in Figure~\ref{fig:intrinsic}.

\begin{table}[h]
\centering
\caption{Mean best validation NMSE for the intrinsic dimension 
experiment with $T = 64$,
averaged over $16$ random seeds.}
\label{tab:intrinsic}
\begin{tabular}{|c|c|c|c|}
\hline
$D$ & $(D,D)$ & $(D-1,D)$ & $(D,D-1)$ \\
\hline
$2$ & $5.89 \times 10^{-5}$ & $5.89 \times 10^{-2}$ & $3.67 \times 10^{-2}$ \\
$3$ & $6.37 \times 10^{-5}$ & $2.64 \times 10^{-2}$ & $2.07 \times 10^{-2}$ \\
$4$ & $6.92 \times 10^{-5}$ & $1.39 \times 10^{-2}$ & $1.30 \times 10^{-2}$ \\
$5$ & $4.69 \times 10^{-4}$ & $4.46 \times 10^{-3}$ & $7.19 \times 10^{-3}$ \\
$6$ & $1.99 \times 10^{-3}$ & $6.26 \times 10^{-3}$ & $7.67 \times 10^{-3}$ \\
\hline
\end{tabular}
\end{table}

Across all values of $D$,
the $(D,D)$ configuration achieves substantially lower NMSE than the 
$(D-1,D)$ and $(D,D-1)$ configurations,
with a gap of two to three orders of magnitude for $D \le 5$.
For $D \ge 6$,
the gap narrows but remains consistent,
which we attribute to the increased difficulty of training with 
larger head counts within the fixed sequence length $T = 64$.

\subsection{Triangle-Center Experiment}
\label{app:exp:triangle}

\paragraph{Target.}
We train two-layer Transformers on the triangle-center target
\begin{equation}
F(X_T) = \min_{1 \le t_1,t_2,t_3 \le T} 
\|x(t_1) + x(t_2) + x(t_3)\|_2^2,
\end{equation}
with token dimension $d = 2$ and sequence lengths 
$T \in \{3, 4, 5, 6, 8, 10, 12, 16, 24, 32, 48, 64\}$.
Each token $x(t) \overset{\mathrm{i.i.d.}}{\sim} \mathcal{N}(0, 2I_2)$, clamping min norm at $0.2$ and max norm at $4$.

\paragraph{Model.}
We train three two-layer Transformer configurations of increasing size:
$(h_1, h_2) = (2, 2)$ with embedding dimension $E = 24$,
$(h_1, h_2) = (4, 4)$ with embedding dimension $E = 24$,
and $(h_1, h_2) = (4, 4)$ with embedding dimension $E = 48$.
The feed-forward hidden width is $w_1 = 32$ for all configurations.
No positional encoding is used.

\paragraph{Training.}
Each configuration and sequence length is trained for $2$ independent 
runs with different random seeds.
We use the AdamW optimizer with learning rate $10^{-3}$ and weight 
decay $10^{-2}$,
with a cosine annealing schedule.
Training is performed on $N = 90{,}000$ i.i.d. samples,
with a held-out validation set of $10{,}000$ samples.
We report the best validation NMSE across training.

\paragraph{Results.}
The mean best validation NMSE for each configuration and sequence 
length is reported in Table~\ref{tab:triangle}.
Results are also shown in Figure~\ref{fig:triangle}.

\begin{table}[h]
\centering
\caption{Mean best validation NMSE for the triangle-center experiment,
averaged over $2$ random seeds.}
\label{tab:triangle}
\begin{tabular}{|c|c|c|c|}
\hline
$T$ & $(2,2)$, $E=24$ & $(4,4)$, $E=24$ & $(4,4)$, $E=48$ \\
\hline
$3$  & $1.31 \times 10^{-3}$ & $9.59 \times 10^{-4}$ & $6.20 \times 10^{-4}$ \\
$4$  & $4.58 \times 10^{-3}$ & $3.81 \times 10^{-3}$ & $3.03 \times 10^{-3}$ \\
$5$  & $1.12 \times 10^{-2}$ & $9.51 \times 10^{-3}$ & $8.41 \times 10^{-3}$ \\
$6$  & $2.12 \times 10^{-2}$ & $2.00 \times 10^{-2}$ & $1.84 \times 10^{-2}$ \\
$8$  & $8.61 \times 10^{-2}$ & $7.20 \times 10^{-2}$ & $7.95 \times 10^{-2}$ \\
$10$ & $2.42 \times 10^{-1}$ & $2.27 \times 10^{-1}$ & $2.13 \times 10^{-1}$ \\
$12$ & $4.32 \times 10^{-1}$ & $4.60 \times 10^{-1}$ & $4.52 \times 10^{-1}$ \\
$16$ & $8.43 \times 10^{-1}$ & $8.53 \times 10^{-1}$ & $8.56 \times 10^{-1}$ \\
$24$ & $9.44 \times 10^{-1}$ & $9.44 \times 10^{-1}$ & $9.44 \times 10^{-1}$ \\
$32$ & $9.57 \times 10^{-1}$ & $9.58 \times 10^{-1}$ & $9.59 \times 10^{-1}$ \\
$48$ & $9.79 \times 10^{-1}$ & $9.79 \times 10^{-1}$ & $9.79 \times 10^{-1}$ \\
$64$ & $9.85 \times 10^{-1}$ & $9.85 \times 10^{-1}$ & $9.85 \times 10^{-1}$ \\
\hline
\end{tabular}
\end{table}

All three configurations exhibit rapidly growing NMSE as $T$ increases,
with NMSE exceeding $0.84$ at $T = 16$ and approaching $1.0$ for 
$T \ge 32$.
The three configurations perform almost identically across all values 
of $T$,
indicating that increasing the number of heads or the embedding 
dimension does not improve approximation within the tested range,
consistent with the \infoflow prediction.
\begin{remark}
    NMSE, namely Normalized Mean Square Error, is equal to $1$ if the model predicts the best constant value regardless of input ($M(X_T) \equiv c$ for constant $c$ independent of $X_T$), which can be achieved by Transformers of any size. NMSE approaching $1$ means that Transformer is not performing better than a constant predictor, which shows that Transformers are not approximating the target efficiently.  
\end{remark}

\clearpage

\section{Explanations for Number of Comparison \ref{sec:NumberOfComparison}}\label{app:NumberOfComparison}

For the generalized $D$-retrieval target 
$F(X_T) = F_0(\oplus_{i=1}^D \max_{1 \le t \le T} f_i(x(t)))$, it has $\beta' = \beta_1 = 1$, and the Number of Comparison of $F$ is upper bounded by $D(T-1)$ and lower bounded by $D(T-D)$.\\
For the target $F(X_T) = \min_{1 \le t_1,t_2,t_3 \le T} \|x(t_1) + x(t_2) + x(t_3)\|_2^2$, we can find that it has $\beta' = \beta_1 = 3$, and its Number of Comparison is upper bounded by $T^3-1$, also lower bounded by $\Omega(T^3)$.

\subsection{Explanations for the Examples}
\paragraph{Generalized $D$-retrieval Target} We first show the upper bound by the following construction, which is intuitive. 
\begin{enumerate}
    \item It is clear that $|I_F(X_T)| \le D, a.e.$, so we construct $D$ Trees of Comparison. 
    \item For $j=1, \dots, D$, $\mathcal{T}_{j}$ have exactly $T$ leaves, each leaf being a single-element index set $\{t\}$ for $t=1, \dots, T$. And $f_{in}(x) = f_j(x)$ for all internal nodes of $\mathcal{T}_{j}$. Thus $I_{fi}(X_T, \mathcal{T}_{j}) = \argmax_{1 \le t \le T} f_j(x(t))$. 
    \item We then have $I_F(X_T) = \bigcup_{j=1}^{T}I_{fi}(X_T, \mathcal{T}_{j})$. 
    \item These $D$ Trees of Comparison each has $T$ leaves, thus it upper bounds the Number of Comparison of $F$ by $D(T-1)$. 
\end{enumerate}
Assuming that each $f_i(z)$ obtains maximum at different $z_i$, and non-degeneracy condition as in \cite{yu2026.EffectAttentionHead}. We then show that the Number of Comparison is lower bounded by $D(T-D)$. 
\begin{enumerate}
    \item It is clear that $\beta_1 = 1$. As all $t$'s can appear in $I_F(X_T)$, so $\{t\}$ has to be some leaf of some Tree. As there are at most $D$ Trees, We have $N' \ge T-D$, thus $\beta' = 1$.  
    \item As $|I_F(X_T)| = D, a.e.$ and $|I_{fi}(X_T, \mathcal{T}_{j})|=1$, thus we need exactly $D$ Trees of Comparisons. 
    \item For each $\mathcal{T}_j$, denote by $I_j \subseteq [T]$ the union of index sets on all its leaves. And Number of Comparison is calculated with $\sum_{j=1}^D (|I_j|-1)$. 
    \item Assume that $|I_j| \le T-D$ for some $j$. Then suppose $t_1, \dots, t_D \notin I_j$. Then let $x(t_i) = z_i$ for $i=1, \dots, D$, we have that $I_F(X_T) = \{t_1, \dots, t_D\}$. However, $[I_F(X_T) \cap I_{fi}(X_T, \mathcal{T}_j)] \subseteq [I_F(X_T) \cap I_j] = \varnothing$. Then we have that 
    \begin{align}
        &|I_F(X_T) \cap \bigcup_{i=1}^{T}I_{fi}(X_T, \mathcal{T}_{i}) |\\
        =& |I_F(X_T) \cap \bigcup_{i\ne j}I_{fi}(X_T, \mathcal{T}_{i})| \\
        \le & |\bigcup_{i\ne j}I_{fi}(X_T, \mathcal{T}_{i})|\\
        \le & D-1 \\
        <& |I_F(X_T)|
    \end{align}
    which contradicts with $|I_F(X_T) = \bigcup_{i=1}^TI_{fi}(X_T, \mathcal{T}_{i})|$. 
    \item Then we have $|I_j|\ge T-D+1$ for all $j=1, \dots, D$, and thus $N' \ge D(T-D+1) -D = D(T-D)$. 
\end{enumerate}

We then successfully showed that the Number of Comparison of $F(X_T) = F_0(\oplus_{i=1}^D \max_{1 \le t \le T} f_i(x(t)))$ is $DT+O(1)$. 

\paragraph{Triangle-Center Target}\label{app:Triangle_explain} For triangle-center target $F(X_T) = \min_{1 \le t_1,t_2,t_3 \le T} \|x(t_1) + x(t_2) + x(t_3)\|_2^2$, we show that it has $\beta_1 = 3$ and $\beta' = 3$, with Number of Comparison upper bounded by $T^3-1$ and lower bounded by $\Omega(T^3)$ (Or more specifically, $\frac{T^3}{6}+ O(T^2)$). (For this simplicity we consider $\mathcal{X} = [-1,1]^d$ for this case, where similar results can be transformed back to $\mathcal{X}^T = [0,1]^T$.)

We first show that its Number of Comparison is upper bounded by $T^3-1$. 
\begin{enumerate}
    \item We construct one Tree of Comparison, with each leaf being an ordered set $(t_1, t_2, t_3)$ that allows repetition. There are $T^3$ choices of such ordered set. 
    \item For all internal nodes, they have the same $f_{in}(X[I]) = -\|\sum_{t \in I} x(t)\|_2^2$. Thus at the root, we have $I_{fi}(X_T, \mathcal{T}) = \argmax_{I \subseteq [T], |I|=3} \|\sum_{t \in I} x(t)\|_2^2 = I_F(X_T)$. 
\end{enumerate}
With the above construction, we have that the Order and Dimension of Comparison upper bounded by $3$, namely $\beta' \le \beta_1 \le 3$. 

We then show the lower bound. Consider $\mathcal{X}^T = [-1,1]^{Td} \subset \mathbb{R}^{Td}$ as part of the $Td$-dimensional Euclidean Space. Now for two given triple pairs (unordered) of indices $I_0 = (t_1, t_2, t_3)$ and $J_0 = (s_1, s_2, s_3)$ where $t_i, s_j$ are pairwise distinct (We assume $T>6$). And denote by $f$ the function $f(X[I]) = \|\sum_{t \in I} x(t)\|_2^2$ for index set $I$. We define the following sets: 
\begin{enumerate}
    \item $\Gamma'' = \{X_T \in \mathcal{X}^T : f(X[I_0]), f(X[J_0]) < f(X[I']) \text{ for all } I'\ne I_0, J_0;\, |I'|=3\}$.
    $\Gamma$ is the set where $X[I_0]$ and $X[J_0]$ obtain the minimum two for $f$. Under the uniform measure $\mu$ on $\mathcal{X}^T = [-1,1]^{Td}$, we have $\mu(\Gamma'') >0$. 
    \item $\Gamma''_l = \{X_T \in \mathcal{X}^T : f(X[I_0])< f(X[J_0]) < f(X[I']) \text{ for all } I'\ne I_0, J_0;\, |I'|=3\}$. \\
    $\Gamma''_r = \{X_T \in \mathcal{X}^T : f(X[J_0])< f(X[I_0]) < f(X[I']) \text{ for all } I'\ne I_0, J_0;\, |I'|=3\}$.\\
    We have that $\Gamma''_r + \Gamma''_l = \Gamma''$; $\mu(\Gamma''_l) = \mu(\Gamma''_r) = \frac{\mu(\Gamma'')}{2} >0$. 
    \item $\gamma'' = \{X_T \in \mathcal{X}^T : f(X[I_0])= f(X[J_0]) < f(X[I']) \text{ for all } I'\ne I_0, J_0;\, |I'|=3\} \subset \Gamma''$ is a $Td-1$ dimensional manifold in $\Gamma''$. As $\gamma''$ is in $\mathbb{R}^{Td}$, we can define the area measure $\nu$ on $\gamma''$, and we have $\nu (\gamma'') >0$. 
    \item By symmetry, we can find a connected component $\Gamma \subseteq \Gamma''$ that is separated into two pieces $\Gamma_l, \Gamma_r$ such that $\mu(\Gamma) = 2\mu(\Gamma_l) = 2 \mu(\Gamma_r) >0$, and $\gamma = \gamma'' \cap \Gamma$ also have $\nu(\gamma)>0$. 
\end{enumerate}
Now for each pairs of unordered triple $I', J'$ that is not the same as $I_0, J_0$, let $\gamma' = \{X_T \in \mathcal{X}^T: g(X[I_1]) = g(X[J_1])\} \cap \Gamma$ for non-degenerate $g$ (namely $\gamma'$ have dimension $Td-1$). As $\{I', J'\} \ne \{I_0, J_0\}$, we have that $\gamma' \cap \gamma$ is a manifold of dimension $Td-2$ (As we cannot have  $I_0 \subseteq I', J_0 \subseteq J'$ together), thus $\nu(\gamma' \cap \gamma)=0$. Suppose each $\gamma'$ separates $\Gamma$ into $\Gamma_l'$ and $\Gamma_r'$, and we have a finite number $M$ of such pairs $(I',J',g, \gamma', \Gamma_l', \Gamma_r')$, we have that $\nu [(\cup \gamma') \cap \gamma] = 0$.  \\
By continuity and and compactness of each $\gamma'$, as $\cup \gamma'$ cannot interpolate $\gamma$ in a dense manner, thus if we use all the $\gamma'$'s to split $\Gamma$ into (at most) $2^M$ subsets (In the sense whether a point in $\Gamma$ lies on the "left" or "right" of $\gamma'$, and one such subset is for points on the same side for all $\gamma'$.) We call these subsets $\Gamma_1, \dots, \Gamma_K$ for $K=2^M$. Then we have that $\min_{I_k \subseteq[K]} \mu (\Gamma_l \triangle (\bigcup_{i \in I_k} \Gamma_i)) = \min_{I_k \subseteq[K]} \mu (\Gamma_r \triangle (\bigcup_{i \in I_k} \Gamma_i)) >0$. \\
This shows that we cannot use many comparisons of other pairs $(I',J')$ to exactly characterize the comparison of $(I_0, J_0)$. Thus when $\beta_1 = 3$, and that the Tree of Comparison have finite leaves, (Thus finite number of possible pairwise comparison performed, where $I_{fi}(X_T, \mathcal{T})$ depends only on the comparison results in the form of separation.) $I_0$ and $J_0$ need to be on the leaves of a Tree of Comparison, namely the comparison of $I_0, J_0$ should be among the possible pairwise comparisons. \\
By extending this argument to all $I_0, J_0$ pair, we have that all unordered $I_0=\{t_1, t_2, t_3\}$ with distinct $t_1, t_2, t_3$ should be on the leaf of at least one Tree of Comparison. This gives a total number of $\binom{T}{3}$ leaves. With $|I_F(X_T)| \le D_0=3,a.e.$, we have the Number of Comparison of $F$ lower bounded by $\binom{T}{3}-3$. Thus the Number of Comparison of $F$ is lower bounded by $\frac{T^3}{6} + O(T^2)$.

\begin{figure}[h]
\centering
\begin{subfigure}[t]{0.48\linewidth}
\centering
\begin{tikzpicture}[
    tnode/.style={
        circle,
        draw,
        minimum size=0.7cm,
        font=\small
    },
    label/.style={
        font=\footnotesize,
        align=center
    },
    arr/.style={
        -,
        thick
    }
]

\node[tnode] (l1) at (0.0, 3.0) {};
\node[tnode] (l2) at (1.2, 3.0) {};
\node[tnode] (l3) at (2.4, 3.0) {};
\node[tnode] (l4) at (3.6, 3.0) {};

\node[label] at (0.0, 3.75) {$\{1\}$};
\node[label] at (1.2, 3.75) {$\{2\}$};
\node[label] at (2.4, 3.75) {$\{3\}$};
\node[label] at (3.6, 3.75) {$\{4\}$};

\node[tnode] (m1) at (0.6, 1.5) {};
\node[tnode] (m2) at (3.0, 1.5) {};

\node[label, right=0.05cm of m1] {$\{2\}$};
\node[label, right=0.05cm of m2] {$\{3\}$};

\node[tnode] (r) at (1.8, 0.0) {};
\node[label, right=0.05cm of r] {$\{3\}$};

\draw[arr] (l1) -- (m1);
\draw[arr] (l2) -- (m1);
\draw[arr] (l3) -- (m2);
\draw[arr] (l4) -- (m2);

\draw[arr] (m1) -- (r);
\draw[arr] (m2) -- (r);

\node[label] at (0.3, 0.0) {Root};
\draw[->] (0.75, 0.0) -- (r);

\end{tikzpicture}
\caption{$\mathcal{T}_1$ with $f_{\{in\}}(x) = x$ at all internal nodes,
finding the maximum.
The root gives $\{3\} = I_{\{f\}}(X_T, \mathcal{T}_1)$,
the position of the maximum value.}
\label{fig:tr1}
\end{subfigure}
\hfill
\begin{subfigure}[t]{0.48\linewidth}
\centering
\begin{tikzpicture}[
    tnode/.style={
        circle,
        draw,
        minimum size=0.7cm,
        font=\small
    },
    label/.style={
        font=\footnotesize,
        align=center
    },
    arr/.style={
        -,
        thick
    }
]

\node[tnode] (l1) at (0.0, 3.0) {};
\node[tnode] (l2) at (1.2, 3.0) {};
\node[tnode] (l3) at (2.4, 3.0) {};
\node[tnode] (l4) at (3.6, 3.0) {};

\node[label] at (0.0, 3.75) {$\{1\}$};
\node[label] at (1.2, 3.75) {$\{2\}$};
\node[label] at (2.4, 3.75) {$\{3\}$};
\node[label] at (3.6, 3.75) {$\{4\}$};

\node[tnode] (m1) at (0.6, 1.5) {};
\node[tnode] (m2) at (3.0, 1.5) {};

\node[label, right=0.05cm of m1] {$\{1\}$};
\node[label, right=0.05cm of m2] {$\{4\}$};

\node[tnode] (r) at (1.8, 0.0) {};
\node[label, right=0.05cm of r] {$\{1\}$};

\draw[arr] (l1) -- (m1);
\draw[arr] (l2) -- (m1);
\draw[arr] (l3) -- (m2);
\draw[arr] (l4) -- (m2);

\draw[arr] (m1) -- (r);
\draw[arr] (m2) -- (r);

\node[label] at (0.3, 0.0) {Root};
\draw[->] (0.75, 0.0) -- (r);

\end{tikzpicture}
\caption{$\mathcal{T}_2$ with $f_{\{in\}}(x) = -x$ at all internal nodes,
finding the minimum.
The root gives $\{1\} = I_{\{f\}}(X_T, \mathcal{T}_2)$,
the position of the minimum value.}
\label{fig:tr2}
\end{subfigure}
\caption{Two comparison trees for the target 
$F(X_T) = \max_{1\le t\le T} x(t) + \min_{1\le t\le T} x(t)$ with $D = 2$,
illustrated on the input $x(1)=0.1$,
$x(2)=0.3$,
$x(3)=0.4$,
$x(4)=0.2$ with $T=4$.
Each leaf corresponds to one input position,
and each internal node retains the index of the winning position after comparison.
The root of each tree gives the active index set $I_{fi}(X_T, \mathcal{T})$.}
\label{fig:comparison-trees}
\end{figure}
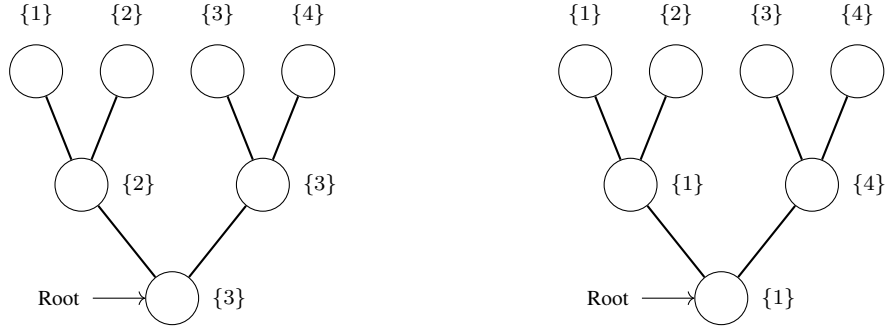

\clearpage
\section{Figures}

\begin{figure}[h]
\centering
\begin{tikzpicture}[scale=0.9,
    token/.style={
        circle,
        draw,
        minimum size=0.7cm,
        font=\small
    },
    token highlight blue/.style={
        circle,
        draw,
        minimum size=0.7cm,
        font=\small,
        fill=blue!15
    },
    token highlight green/.style={
        circle,
        draw,
        minimum size=0.7cm,
        font=\small,
        fill=green!20
    },
    token out/.style={
        circle,
        draw,
        minimum size=0.7cm,
        font=\small,
        fill=gray!15
    },
    faded/.style={
        circle,
        draw=gray!40,
        text=gray!40,
        minimum size=0.7cm,
        font=\small
    },
    arr/.style={
        ->,
        thick
    },
    label/.style={
        font=\footnotesize,
        align=center
    }
]

% Subfigure I

\def\Ix{0}

\node[faded]                (I1) at (\Ix+0.0, 2.0) {$1$};
\node[token highlight blue] (I2) at (\Ix+1.2, 2.0) {$2$};
\node[token]                (I3) at (\Ix+2.4, 2.0) {$3$};
\node[faded]                (I4) at (\Ix+3.6, 2.0) {$4$};
\node[token out]            (Io) at (\Ix+2.4, 0.0) {$3'$};

\draw[arr, blue, line width=1.5pt] (I2) -- (Io);
\draw[arr, orange, line width=1.2pt] (I3) -- (Io);
\draw[->, gray!30, dashed] (I1) -- (Io);
\draw[->, gray!30, dashed] (I4) -- (Io);

\node[label, blue]   at (\Ix+1.2, 2.75) {argmax};
\node[label, orange] at (\Ix+2.5, 2.82)  {residual};
\node[label]         at (\Ix+1.8, -0.9) {\textbf{I}: Max-position};

%  Subfigure II

\def\IIx{5.2}

\node[token] (II1) at (\IIx+0.0, 2.0) {$1$};
\node[token] (II2) at (\IIx+1.2, 2.0) {$2$};
\node[token] (II3) at (\IIx+2.4, 2.0) {$3$};
\node[token] (II4) at (\IIx+3.6, 2.0) {$4$};
\node[token out] (IIo) at (\IIx+2.4, 0.0) {$3'$};

\draw[arr, red!70!black, line width=1.2pt] (II1) -- (IIo);
\draw[arr, red!70!black, line width=1.2pt] (II2) -- (IIo);
\draw[arr, red!70!black, line width=1.2pt] (II3) -- (IIo);
\draw[arr, red!70!black, line width=1.2pt] (II4) -- (IIo);

\node[label, red!70!black] at (\IIx+1.8, 2.75) {all positions};
\node[label] at (\IIx+1.8, -0.9) {\textbf{II}: Global aggregation};

% Subfigure III

\def\IIIx{10.4}

\node[token highlight green] (III1) at (\IIIx+0.0, 2.0) {$1$};
\node[faded]                 (III2) at (\IIIx+1.2, 2.0) {$2$};
\node[token highlight green] (III3) at (\IIIx+2.4, 2.0) {$3$};
\node[faded]                 (III4) at (\IIIx+3.6, 2.0) {$4$};
\node[token out]             (IIIo) at (\IIIx+2.4, 0.0) {$3'$};

\draw[arr, green!50!black, line width=1.5pt] (III1) -- (IIIo);
\draw[arr, green!50!black, line width=1.5pt] (III3) -- (IIIo);
\draw[->, gray!30, dashed] (III2) -- (IIIo);
\draw[->, gray!30, dashed] (III4) -- (IIIo);

\node[label, green!50!black] at (\IIIx+1.2, 2.75) {fixed by PE};
\node[label] at (\IIIx+1.8, -0.9) {\textbf{III}: Specific position};

\draw[gray!30, rounded corners]
    (-0.5, -0.5) rectangle (4.2, 3.2);
\draw[gray!30, rounded corners]
    (4.6, -0.5) rectangle (9.3, 3.2);
\draw[gray!30, rounded corners]
    (9.7, -0.5) rectangle (14.5, 3.2);

\end{tikzpicture}
\caption{Illustration of the three information propagation mechanisms in attention layers,
shown for a sequence of four tokens.
The output token $3'$ is shown at the bottom of each subfigure,
with arrows indicating which input tokens contribute to it.
\textbf{I}: Max-position retrieval aggregates information from the argmax position (blue, shaded) via attention and from its own position via the residual connection (orange).
\textbf{II}: Global aggregation aggregates information from all input positions.
\textbf{III}: Specific position aggregation aggregates information from a fixed set of positions (green, shaded) determined by positional encoding alone,
independent of token content.
Faded tokens with dashed arrows indicate positions that do not contribute.}
\label{fig:mechanisms}
\end{figure}
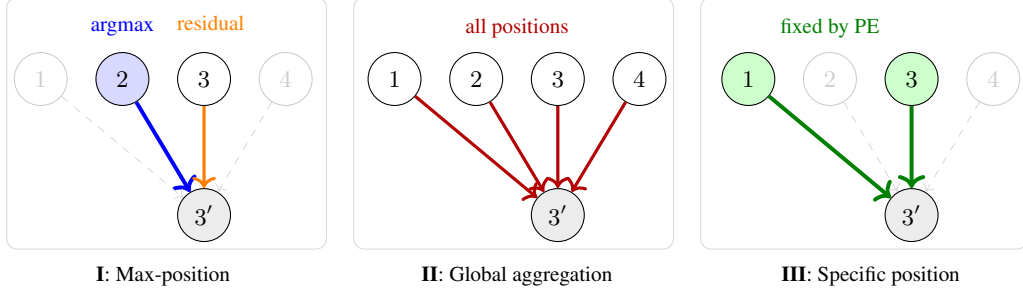

\begin{figure}[h]
\centering
\includegraphics[width=0.6\linewidth]{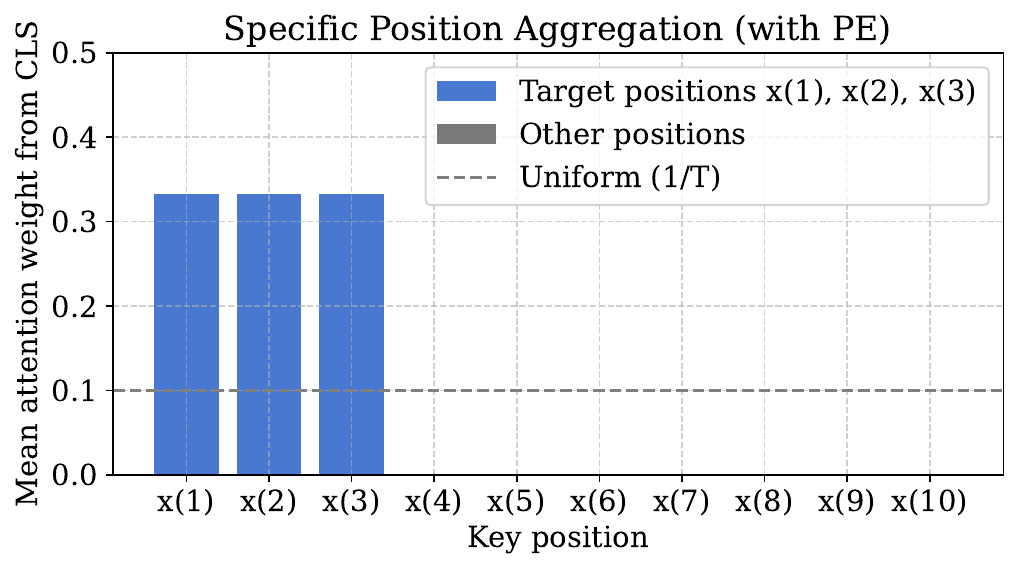}
\caption{Mean CLS attention weight per input position,
averaged over the validation set,
for a Transformer trained with positional encoding on $F(X_T) = x(1) + x(2) + x(3)$.
Attention concentrates entirely on positions $1$, $2$, $3$ with weight $1/3$ each,
while all other positions receive zero weight,
confirming position-specific selection independent of token content.}
\label{fig:pos-agg}
\end{figure}

% \begin{figure}[h]
% \centering
% \begin{subfigure}[t]{0.48\linewidth}
% \centering
% \includegraphics[width=\linewidth]{fig/intrinsic_dimension.pdf}
% \caption{Best validation NMSE for two-layer Transformers trained on
% the target \eqref{eq:intrinsic-target} with $T = 64$, $D = 2, 3, 4, 5, 6$ and head configurations $(D,D)$,
% $(D-1,D)$,
% $(D,D-1)$.
% A sharp phase transition occurs at $h_1 = h_2 = D$,
% with $(D,D)$ achieving NMSE on the order of $10^{-5}$ while
% configurations with fewer heads remain at $10^{-3}$ to $10^{-2}$,
% confirming the \infoflow prediction that $D$ is the intrinsic
% dimension of the target.}
% \label{fig:intrinsic}
% \end{subfigure}
% \hfill
% \begin{subfigure}[t]{0.5\linewidth}
% \centering
% \includegraphics[width=\linewidth]{fig/triangle_center.pdf}
% \caption{Best validation NMSE against sequence length $T$ for
% two-layer Transformers of three different sizes trained on the
% triangle-center target \eqref{eq:triangle-target}.
% All configurations exhibit rapidly growing NMSE with $T$,
% with all three curves nearly coinciding,
% consistent with the \infoflow prediction that the parameter cost
% of approximating $F$ grows with $T$ for any fixed architecture,
% regardless of model size.}
% \label{fig:triangle}
% \end{subfigure}
% \caption{
% Experimental validation of the intrinsic dimension phenomenon (left) 
% and the approximation hardness of the triangle-center task (right).}
% \label{fig:predictions}
% \end{figure}

\clearpage
\paragraph{Illustrative example of \infoflow.}\label{app: example}
We illustrate the two-step procedure on a concrete example.
Consider the target $F(X_T) = \min_{1 \le s,t \le T} x(t)^\top x(s)$
and a two-layer Transformer with one head per layer.
Let the input sequence be $x(1) = (0,-1)$,
$x(2) = (0.7, 0.7)$,
$x(3) = (0,1)$,
$x(4) = (-0.2,-0.9)$,
so that $T = 4$.

We apply update rule (1) max position aggregation at both layers.
At layer $0$,
the information sets are initialized as $I(t,0) = \{t\}$ for $t = 1,2,3,4$ and $I(5,0) = \varnothing$.
For the first layer,
letting the attention score function be $A_{1,1}(X[I(t,0)], X[I(s,0)]) = \min_{t' \in I(t,0), s' \in I(s,0)} x(t')^\top x(s')$,
we compute the argmax positions
\begin{equation}
s_{1,0}(1) = 3, \quad s_{1,0}(2) = 4, \quad s_{1,0}(3) = 1, \quad s_{1,0}(4) = 3.
\end{equation}
Applying the max position aggregation rule gives
\begin{equation}
I(1,1) = \{1,3\}, \quad I(2,1) = \{2,4\}, \quad I(3,1) = \{1,3\}, \quad I(4,1) = \{3,4\}.
\end{equation}
For the second layer,
the classification token at position $5$ reads out from the sequence.
Letting the attention score function be $A_{2,1}(X[I(5,1)], X[I(s,1)]) = \min_{s_1,s_2 \in I(s,1)} x(s_1)^\top x(s_2)$,
we obtain $s_{1,1}(5) = 1$,
and therefore
\begin{equation}
I(5,2) = I(s_{1,1}(5), 1) = I(1,1) = \{1,3\}.
\end{equation}
The active index set of the target is $I_F(X_T) = \{1,3\}$,
since the minimum $\min_{s,t} x(t)^\top x(s)$ is attained at $(s,t) = (1,3)$. For this $X_T$ we have $I_F(X_T) \subseteq I(5,2)$. 
%As $I_F(X_T) = \{1,3\} \subseteq I(5,2) = \{1,3\}$,
%the \infoflow model learns the target $F$,
%confirming that the two-layer one-head-per-layer Transformer can approximate $F$ effectively under max position aggregation.
The evolution of information sets across layers is illustrated in 
Figure~\ref{fig:infoflow-example}.

\begin{figure}[h]
\centering
\begin{tikzpicture}[scale=0.6,
    node distance = 1.8cm,
    token/.style = {
        circle,
        draw,
        minimum size = 0.9cm,
        font = \small
    },
    token highlight/.style = {
        circle,
        draw,
        minimum size = 0.9cm,
        font = \small,
        fill = red!20
    },
    cls/.style = {
        circle,
        draw,
        minimum size = 0.9cm,
        font = \small,
        dashed
    },
    cls highlight/.style = {
        circle,
        draw,
        minimum size = 0.9cm,
        font = \small,
        dashed,
        fill = blue!10
    },
    arr/.style = {
        ->,
        thick,
        blue
    },
    label/.style = {
        font = \footnotesize,
        align = center
    }
]

% Layer 0 nodes
\node[token highlight] (t10) at (0, 4)   {$1$};
\node[token]           (t20) at (0, 2.5) {$2$};
\node[token highlight] (t30) at (0, 1)   {$3$};
\node[token]           (t40) at (0,-0.5) {$4$};
\node[cls]             (t50) at (0,-2)   {CLS};

% Layer 1 nodes
\node[token] (t11) at (4, 4)   {$1$};
\node[token] (t21) at (4, 2.5) {$2$};
\node[token] (t31) at (4, 1)   {$3$};
\node[token] (t41) at (4,-0.5) {$4$};
\node[cls]   (t51) at (4,-2)   {CLS};

% Layer 2 nodes
\node[cls highlight] (t52) at (8,-2) {CLS};

% Column labels
\node[label] at (0, 5.2)  {Layer $0$};
\node[label] at (4, 5.2)  {Layer $1$};
\node[label] at (8, 5.2)  {Layer $2$};

% Information set labels layer 0
\node[label, left=0.1cm of t10] {$\{1\}$};
\node[label, left=0.1cm of t20] {$\{2\}$};
\node[label, left=0.1cm of t30] {$\{3\}$};
\node[label, left=0.1cm of t40] {$\{4\}$};
\node[label, left=0.1cm of t50] {$\varnothing$};

% Information set labels layer 1
\node[label, right=0.1cm of t11] {$\{1,3\}$};
\node[label, right=0.1cm of t21] {$\{2,4\}$};
\node[label, right=0.1cm of t31] {$\{1,3\}$};
\node[label, right=0.1cm of t41] {$\{3,4\}$};

% Information set labels layer 2
\node[label, right=0.1cm of t52] {$\{1,3\} = I_F(X_T)$};

% Active index set label at layer 0
% \node[red, label] at (-0.8, 4.5) {$I_F(X_T)$};

% Residual connections
\draw[arr, gray] (t10) -- (t11);
\draw[arr, gray] (t20) -- (t21);
\draw[arr, gray] (t30) -- (t31);
\draw[arr, gray] (t40) -- (t41);

% Argmax arrows layer 0 to layer 1
\draw[arr] (t30) to[bend left=20]  (t11);
\draw[arr] (t40) to[bend left=20]  (t21);
\draw[arr] (t10) to[bend right=20] (t31);
\draw[arr] (t30) to[bend right=10] (t41);

% Argmax arrow layer 1 to layer 2
\draw[arr] (t11) to[bend right=30] (t52);

\end{tikzpicture}
\caption{Illustration of InfoFlow on the target $F(X_T) = \min_{1 \le s,t \le T} x(t)^\top x(s)$ with a two-layer one-head-per-layer Transformer.
Each node shows the information set $I(t,l)$ at position $t$ and layer $l$.
Blue arrows indicate argmax information flow selected by each attention head.
Gray arrows indicate residual connections preserving the token's own information.
Red-shaded nodes at layer $0$ indicate the active index set $I_F(X_T) = \{1,3\}$.
After two layers,
the classification token accumulates $I(5,2) = \{1,3\} = I_F(X_T)$,
confirming that the InfoFlow model learns the target.}
\label{fig:infoflow-example}
\end{figure}
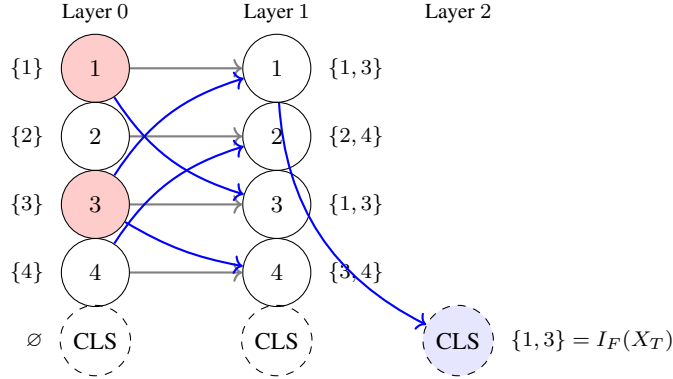

\clearpage

\end{document}